\documentclass[twoside]{article}

\usepackage{url}
\usepackage{graphicx}
\usepackage[font=small]{caption}
\usepackage{amsmath}

\usepackage{amsfonts}
\usepackage{subcaption}
\usepackage{graphicx} 
\usepackage[round,sort]{natbib}

\usepackage{xcolor}
\usepackage{algorithm,algorithmic}
\usepackage{tabularx} 
\usepackage{amsthm}
\usepackage{amsfonts}
\usepackage{amsmath}
\usepackage{amssymb}
\usepackage{mathrsfs}
\usepackage{thmtools,thm-restate}
\usepackage{bbm}
\usepackage[normalem]{ulem}

\newcommand{\eq}[1]{\begin{align*}#1\end{align*}}

\newcommand{\elist}[1]{\begin{enumerate}#1\end{enumerate}}

\newtheorem{definition}{Definition}
\newtheorem{corollary}{Corollary}

\newtheorem{proposition}{Proposition}
\newtheorem{lemma}{Lemma}
\newcounter{assump}
\newtheorem{assumption}[assump]{Assumption}

\DeclareMathOperator*{\argmax}{arg\,max}
\DeclareMathOperator*{\argmin}{arg\,min}

\DeclareMathOperator*{\unif}{Unif}

\DeclareMathOperator*{\poly}{poly}
\DeclareMathOperator*{\ber}{Ber}

\usepackage{enumitem}
\usepackage{wrapfig}

\newcommand{\beq}{\begin{equation}}

\newcommand{\eeq}{\end{equation}}
\newcommand{\beqs}{\begin{equation*}}
\newcommand{\eeqs}{\end{equation*}}

\renewcommand{\AA}{\mathcal{A}}

\newcommand{\B}{\mathbf{B}}
\newcommand{\BB}{\mathcal{B}}

\newcommand{\SSS}{\mathcal{S}}

\newcommand{\R}{\mathbb{R}}
\newcommand{\RR}{\mathcal{R}}
\newcommand{\VV}{\mathcal{V}}

\newcommand{\EE}{\mathcal{E}}

\newcommand{\FF}{\mathcal{F}}

\newcommand{\UU}{\mathcal{U}}
\newcommand{\GG}{\mathcal{G}}
\newcommand{\MM}{\mathcal{M}}
\newcommand{\N}{\mathbb{N}}

\newcommand{\regret}{\text{Regret}}
\newcommand{\E}{\mathbb{E}}

\newcommand{\Tau}{\mathcal{T}}

\newcommand{\ZZ}{\mathcal{Z}}

\newcommand{\1}{\mathbf{1}}
\newcommand{\<}{\langle}
\renewcommand{\>}{\rangle}

\newcommand{\abs}[1]{\ensuremath{| #1 |}}

\newcommand{\ceil}[1]{\lceil #1 \rceil}

\newcommand{\ms}{\textsf{ECE}}
\newcommand{\msgap}{\textsf{ECE-Gap}}

\newcommand{\WW}{\mathcal{W}}

\newcommand{\T}{\mathsf{T}}
\renewcommand{\Pr}{\mathbb{P}}

\renewcommand{\T}{\mathcal{T}}

\newcommand{\tempdelete}[1]{}

\usepackage[accepted]{aistats2021}
\begin{document}

%
\runningtitle{Online Model Selection for Reinforcement Learning}

%
\runningauthor{Lee, Pacchiano, Muthukumar, Kong, Brunskill}

\twocolumn[

\aistatstitle{Online Model Selection for Reinforcement Learning with Function Approximation}

\aistatsauthor{Jonathan N. Lee\And Aldo Pacchiano\And Vidya Muthukumar\And Weihao Kong\And Emma Brunskill}

\aistatsaddress{Stanford University\And  UC Berkeley\And Simons Institute\And University of Washington\And Stanford University}

]

\begin{abstract}

Deep reinforcement learning has achieved impressive successes yet often requires a very large amount of interaction data. This result is perhaps unsurprising, as using complicated function approximation often requires more data to fit, and early theoretical results on linear Markov decision processes provide regret bounds that scale with the dimension of the linear approximation. Ideally, we would like to automatically identify the minimal dimension of the approximation that is sufficient to encode an optimal policy. Towards this end, we consider the problem of model selection in RL with function approximation, given a set of candidate RL algorithms with known regret guarantees. The learner's goal is to adapt to the complexity of the optimal algorithm without knowing it \textit{a priori}. We present a meta-algorithm  that successively rejects increasingly complex models using a simple statistical test. Given at least one candidate that satisfies realizability, we prove the meta-algorithm adapts to the optimal complexity with $\widetilde{O}(L^{5/6} T^{2/3})$ regret compared to the optimal candidate's $\widetilde{O}(\sqrt T)$ regret, where $T$ is the number of episodes and $L$ is the number of algorithms. The dimension and horizon dependencies remain optimal with respect to the best candidate, and our meta-algorithmic approach is flexible to incorporate multiple candidate algorithms and models. Finally, we show that the meta-algorithm automatically admits significantly improved instance-dependent regret bounds that depend on the gaps between the maximal values attainable by the candidates.
\end{abstract}

\section{INTRODUCTION}

Deep reinforcement learning has achieved impressive successes, yet often requires a very large amount of interaction data. This result is perhaps unsurprising, as more complicated function approximations often require  more data to fit. Recent work on theoretical reinforcement learning for some structured function approximation settings has shown regret bounds that scale with a parameter characterizing the complexity of a particular function class. For example, for a type of function approximation by a $d$-dimensional linear model in Markov decision processes (MDPs), prior work has provided bounds that scale as $O(d^{3/2})$ regret~\citep{jin2019provably}, which have been  improved to $O(d)$ even given small inherent Bellman error~\citep{zanette2020learning}.
When the dynamics can be expressed using a matrix, $O(d^{3/2})$ regret bounds have also been provided~\citep{yang2019reinforcement}. 
The choice of dimension $d$ is important: on one hand, if $d$ is under-specified, such regret bounds typically either fail to hold or incur linear regret. 
On the other hand, if $d$ is over-specified, the above regret bounds are unnecessarily large.
Thus, a natural goal is to use the most compact representation suitable to encode the optimal policy for a domain (which we denote as $d_*$).
This optimal representation is typically unknown a priori. 

In this paper we frame this as a model selection question among a set of algorithms with model classes, parameterized by dimensions $\{d \geq 1\}$, that are nested in their regret bound guarantees. 
We assume that at least one class can realize the true underlying domain.
We ask if is there an algorithm that can achieve regret bounds that scale with the minimal realizable model class, given by $d_*$.
Doing so seems subtle: provably efficient reinforcement learning algorithms typically rely heavily on strategic exploration, and using the wrong model class during learning may alias states, resulting in performance that appears strong under the current (incorrect) model class but is actually suboptimal.
Conversely, forced exploration under more complex classes mitigates this problem, but could introduce regret that scales with the more complex model class dependence, even when a simpler model suffices.

Most prior work on model selection for online decision making has focused on contextual bandit settings.
Here, minimax-optimal guarantees were recently shown under eigenvalue assumptions on the features by leveraging the special structure of the stochastic linear contextual bandit setting~\citep{foster2019model,chatterji2020osom}. These results also assume the \textit{knowledge} of a good exploration policy, but such knowledge cannot be relied on in the reinforcement learning setting, where some ``high-reward" states may only be observed under specific, initially unknown sequences of actions.
Slightly weaker model selection guarantees can also be obtained under far more general assumptions by using a \emph{corralling framework} that assumes access to a set of base algorithms, and provides a meta-algorithm that aims to realize the best regret of the (unknown) best algorithm~\citep{agarwal2016corralling,pacchiano2020model,arora2020corralling}.

\paragraph{Our contributions}
We tackle the challenge of model selection in RL under minimal assumptions.
Our main insight is to leverage the knowledge of expected regret that is achievable under a particular model \textit{when it realizes the data}.
Thus, we propose an algorithm in Section~\ref{sec::alg} that maintains a candidate set of model classes at every round, and statistically tests whether each of them is well-specified, or not, by comparing the observed returns under that model class to the regret we should expect from a well-specified model.
Model classes detected as misspecified at any round are permanently eliminated there-after in a manner reminiscent of \textit{active-arm elimination} in the multi-armed bandit problem~\citep{even2006action}; this is a significant simplification over previous meta-algorithms for model selection that were based on adversarial bandit algorithms.
Our choice of action at every round carefully interleaves executing the candidate model class of minimal complexity with executing algorithms using higher-order models.
This procedure is shown to automatically satisfy the needed exploration-exploitation trade-off for model selection.
In Section~\ref{sec::result}, we show the regret bounds exactly match the model complexity of the unknown best model in $d_*$ (and the finite episode length $H$ in RL), and achieve a $T^{2/3}$ rate when the underlying algorithms have a $T^{1/2}$ rate under minimal assumptions about the underlying dynamics process. This is similar to recent model selection algorithms under general assumptions \citep{pacchiano2020model} which sacrifice either a tight dependence on $T$ or $d_*$. We also demonstrate how our approach is compatible with multiple recently introduced RL results, and provide specific bounds for model selection in such settings. In addition to our algorithm being simpler than a recent model-selection approach~\citep{pacchiano2020model}, we provide new, significantly improved bounds for instances in which there is a constant gap in performance between model classes in Section~\ref{sec::id}. These guarantees are in part \textit{instance-dependent}, as they scale inversely with this performance gap.
From a practical perspective, our wrapper algorithm can be used given any input algorithms with regret guarantees that are nested, which will allow it to directly inherit future advances in provably efficient reinforcement learning.
Finally, the computational complexity of our meta-algorithm only adds an extra factor on the order of the total number of model classes over and above the computational complexity of a single base algorithm.

\section{RELATED WORK}

The problem of model selection in online decision-making environments with limited-information feedback (which includes both bandits and reinforcement learning), has been an active area of recent research~\citep{agarwal2016corralling,foster2019model,chatterji2020osom,pacchiano2020model} and poses challenges that are both statistical and algorithmic.

\paragraph{Nearly Optimal Online Model Selection}
The best available guarantees for online model selection have been obtained for the linear contextual bandits setting~\citep{chu2011contextual,abbasi2011improved}.
Here, the best worst-case bound when the optimal model class is given is of the form $\mathcal{O}(\sqrt{d_*T})$, where $d_*$ is the dimension of the minimal feature space that realizes the data and $T$ is the total number of rounds: in model selection, several models with different $d$ are provided and the minimal $d_*$ is unknown.
When the contextual information is stochastic,~\cite{foster2019model} obtain model selection guarantees of the form $\mathcal{O}(d_*^{1/3} T^{2/3})$ under an action-averaged eigenvalue condition, and~\cite{chatterji2020osom} match the optimal guarantee when choosing between multi-armed bandits and contextual bandits under a stronger universal eigenvalue condition that ensures that contexts corresponding to all arms are sufficiently diverse.
The results of~\cite{foster2019model} leverage the fact that it is possible to estimate the optimal value under the optimal model (what we will denote as $V^*$ in this paper) at a faster rate of $\sqrt{d}/n$ as compared to finding the optimal policy under the complex model (which has estimation error rate $d/n$).
Both critically leverage both stochasticity of contextual information and linearity of the model.
These bandit approaches also rely on \textit{a priori} access to a policy that explores the environment and allows for off-policy estimation. %
However, reward-free exploration in RL~\citep{jin2020reward,wang2020reward,zanette2020learning} can sometimes be as or more complex than estimating the optimal policy.

Though there has been some work on offline feature selection and model selection for RL given a batch of data (see e.g.~\cite{parr2008analysis,jiang2015abstraction,hallak2013model,farahmand2011model}), there has been very little work specifically on online model selection in reinforcement learning. Prior work provided PAC results for online feature selection for factored tabular MDPs~\citep{guo2018Sample}. More recent work provides regret bounds~\citep{abbasi2020regret} and PAC bounds~\citep{modi2020sample} for model selection in online RL when the optimal value $V^*$ is given: however, unlike in contextual bandits~\citep{foster2019model,kong2020sublinear}, there are no known algorithms for estimating $V^*$ faster than identifying the optimal policy in RL settings.  

\paragraph{Corralling Methods}

Other researchers have provided general-purpose meta-algorithms designed for model selection for bandit settings that yield weaker, but still non-trivial and interesting statistical guarantees of the form $\mathcal{O}(\RR_*^{\alpha} T^{\beta})$ for arbitrary $\alpha \geq 1, \beta < 1$, where $\RR_*$ depends generally on the complexity of the best model class or algorithm and other problem parameters.
The early corralling algorithms for stochastic and adversarial bandits~\citep{agarwal2016corralling},
have recently been simplified and improved 
 under a mild stochastic assumption on the data~\citep{pacchiano2020model}, using a novel smoothing technique broadly applicable to base algorithms with a regret guarantee.
This \textit{stochastic corralling} approach obtains model selection rates with $\alpha = 2, \beta = 1/2$ or $\alpha = 1, \beta = 2/3$ under very general assumptions including the RL setting; however, for technical reasons it still requires a complex two step smoothing procedure to modify the base algorithms to satisfy its regret guarantees.
Our approach recovers rates of the form $\alpha = 1, \beta = 2/3$ (provided in Section~\ref{sec::result}) without sacrificing generality and with a significantly simplified and interpretable algorithm design. 
This simplicity largely arises from using a stochastic master rather than an adversarial master.
As a consequence, our same algorithm can be analyzed to provide significantly stronger model selection guarantees for instances that have a constant gap in performance between model classes; these guarantees are provided in Section~\ref{sec::id}.
Moreover, side information or faster estimators of the optimal value $V^*$, if available, can be naturally incorporated into our design to provide near-optimal rates; see Appendix~\ref{sec::v-star} for precise statements of these guarantees.

\section{SETTING}\label{sec::setting}

We consider the setting of an episodic Markov decision process (MDP) $\MM = (\SSS, \UU, H, r, P, \rho)$, where $\SSS$ and $\UU$ are state and action spaces, $H \in \N$  is the length of an episode, $r = \{ r_h(s_h,u_h) \}$ is the reward function for step $h$ with $r_h(s_h, u_h) \in [0, 1]$, $P = \{ P_h(s_{h+1}|s_h,u_h)\}$ is the transition dynamics for step $h$, and $\rho(s)$ is a fixed initial state distribution. 
A policy maps times and states to actions, $\pi : [H] \times \SSS \to \UU$. 

For a given $h \in [H]$ and $s \in \SSS$, the value function is the expected cumulative reward following policy $\pi$:
\eq{
V^{\pi}_h(s)  := \E_{\pi} \left[  \sum_{h' = h}^H r_{h'}(s_{h'}, u_{h'})  | s_{h} = s\right]
}
and similarly the action-value function is defined as the expected return from first taking action $u$ and then following policy $\pi$:
$
Q^\pi_h(s, u) = r_h(s, u) + \E_{s' \sim P_h(\cdot | s, u)} V_{h + 1}^\pi(s')
$.
The optimal value function is denoted $V^*_h(s) = \sup_{\pi} V^\pi_h(s)$. We write $V^\pi := \E_{s \sim \rho} V^\pi_1(s)$ and denote the optimal value under $\rho$ as $V^* = \sup_{\pi} V^\pi$. 
In this work we primarily evaluate the quality of an algorithm $\AA$ in an MDP $\MM$ by its regret\footnote{Note that regret is here defined with respect to the optimal value. We will also consider algorithms satisfying ``best-in-class" regret guarantees in Section~\ref{sec::id}.} with respect to the (unknown) optimal policy value $V^*$ over $T$ episodes:
\begin{align}\label{eq::regret-v-star}
\regret_T(\AA; \MM) := \sum_{t = 1}^T V^* - V^{\pi_t}.
\end{align}

We are interested in settings where the size of the state space $\mathcal{S}$ and/or action space $\mathcal{U}$ could be very large. 
Hence, we focus on function approximation methods for minimizing regret. A function approximation algorithm takes as input a model class $\FF$ to generalize across states and actions~\citep{agarwal2019reinforcement}. Several natural examples include value-based classes where $\FF: \SSS \times \UU \to \R$ is used to predict action-value functions $Q^\pi$ and model-based classes where $\FF: \SSS \times \UU \times \SSS \to \R$ is used to predict the transition dynamics $P$ and reward $r$.
Concretely, linear MDPs \citep{jin2019provably,yang2019reinforcement} model the transition dynamics as $\<\phi(s, a), \mu(s') \>$, where $\phi \in \R^d$ and $\mu_h$ is a $d$-dimensional vector of measures.

We let $(\AA, \FF)$ denote the pair of algorithm $\AA$ equipped with model class $\FF$. Recent high probability regret (upper) bounds in this setting are sublinear in $T$ and typically  depend polynomially on $d_\FF$, $H$,  and $\log(T/\delta)$, where $d_\FF$ is a measure of statistical complexity of $\FF$ and $\delta \in (0, 1)$ is a failure probability. For example, if $\FF$ is finite, we often have $d_\FF = \log |\FF|$ and if $\FF$ is a class of linear functions of dimension $d$, we have $d_\FF = d$. 
However, provably sublinear regret bounds in $T$ are generally only known for algorithms under problem-specific assumptions for $\FF$---for example, there exists $f^* \in \FF$ such that the function approximation error is $0$. If this condition holds, we say that $\FF$ \textit{realizes} the MDP $\MM$. Conversely, if $\FF$ does not realize $\MM$, then it is \textit{misspecified}. Since we consider settings where $\FF$ may or may not realize $\MM$ and realizability is almost universally assumed among modern RL algorithms with function approximation, we define a general notion of the regret of $\AA$ using $\FF$ under realizability, following \cite{pacchiano2020model}.

\begin{definition}\label{def::realize}
For an MDP $\MM$, let algorithm $\AA$ be equipped with a model class $\FF$. Let $\RR$ be a known function that is $\poly(d_{\FF}, H, \log(T/\delta))$.
The pair $(\AA, \FF)$ is said to be $\RR$-compatible if $\FF$ realizes $\MM$ and we have 
\eq{
\regret_t(\AA; \MM) & \leq \RR(d_{\FF}, H, \log(T/\delta)) \cdot \sqrt{t}.
}
for all $t$ with probability at least $1 - \delta$.
$\RR$ is called a nominal regret coefficient\footnote{It is not necessary that $\RR$ depend only on these arguments; but these arguments are typically of interest in RL regret bounds.} for $(\AA, \FF)$.
\end{definition}

The rationale behind $\RR$-compatible algorithms is the following. For any $(\AA, \FF)$, we may have a regret coefficient $\RR$ in mind (from a provable guarantee) that holds if $\FF$ realizes $\MM$.
The regret $\RR \cdot \sqrt t$ reflects what we hope to achieve if $\FF$ does actually realize $\MM$, and $(\AA, \FF)$ is only defined to be compatible if this happens. 
We remark that realizability is not necessary for a sublinear regret guarantee to hold, but most RL algorithms using function approximation assume it holds, so it is convenient to view both conditions together.

Note that Definition~\ref{def::realize} requires that $\AA$ is anytime, meaning the bound holds at any arbitrary round index $t \in [T]$ even though only the maximal number of rounds, $T$, may be specified. For algorithms without automatic anytime guarantees, this can be remedied up to constant factors via the doubling trick~\citep{cesa2006prediction}. 
We will later give examples of how our model selection algorithm can be used with some recent single task RL algorithms with formal bounds in the function approximation setting.

 \paragraph{Problem Statement}  Here, our goal in model selection is to obtain a regret guarantee that adapts on-the-fly to the model class of minimal complexity that remains competitive with the optimal value. 
 That is, we wish to find the combination of algorithm $\AA$ and model class $\FF$, that is compatible in the sense of Definition~\ref{def::realize}, with the smallest possible leading coefficient $\RR(d_{\FF}, \cdot, \cdot)$. We consider a setting where we are choosing among a set of candidate algorithms $\AA_1, \AA_2, \ldots \AA_L$ with model classes $\{ \FF_{i}\}_{i \in [L]}$, \emph{known} nominal regret coefficients $\{ \RR_i\}_{i \in [L]}$, and complexities $\{ d_i \}_{i \in [L]}$ where $d_i := d_{\FF_i}$ and $\FF_i$ is the model class of $\AA_i$. 
Without loss of generality, we assume the algorithm-model class pairs can be ordered by their regret such that we have
\begin{align}
\label{eq::ordering}
\RR_{i}(d_i, H, \log(T/\delta))  \leq \RR_{i + 1}(d_{i + 1}, H, \log(T/\delta))
\end{align}
for all $i \in [L - 1]$, $T, H \in \N$, and $\delta \in (0, 1)$. For example, if $\{\AA_i\}$ are all instances of the same algorithm that use as input nested model classes $\{ \FF_i\}$, then (\ref{eq::ordering}) is satisfied by ordering $d_1  \leq\ldots \leq  d_L$. 
This naturally captures, among other cases, linear models with nested features \citep{foster2019model}. 
We also assume\footnote{Note that for all other misspecified algorithms, their nominal regret bounds will, in general, not hold. As regret is being measured with respect to $V^*$, it will include the misspecification error terms.} that at least one algorithm is $\RR_i$-compatible for its respective regret coefficient $\RR_i$.  Define $i_* = \min \{ i \in [L] \ : \ (\AA_i, \FF_i) \text{ is }\RR_{i}\text{-compatible}  \}$.

We aim to design a meta-algorithm $\AA$ that selects among $\{\AA_i\}_{i=1}^L$ without knowing $i_*$ \textit{a priori} and, for some $\alpha \geq 0$ and $\beta \in [1/2, 1)$, achieves a guarantee of 
\eq{
\regret_T(\AA) & = O \left( \RR_{i_*}(d_{i_*}, H, \log(T/\delta)) \cdot L^\alpha T^{\beta}  \right).
}

\section{MODEL SELECTION APPROACH}\label{sec::alg}

In this section, we present our model selection meta-algorithm, Explore-Commit-Eliminate (\ms{})
and detail the simple statistical test underlying our approach.

\subsection{Algorithm}
Our meta-algorithm for model selection is described in Algorithm~\ref{alg::ms}.
At a high level, the algorithm proceeds in the following way. 
It takes as input the base algorithms and model classes, their nominal regret coefficients, and their model complexities; mathematically, the input is given by $\{ \AA_i, \FF_i, \RR_i, d_i \}_{i \in [L]}$. 
The number of algorithms $L$, episodes $T \in \N$ and failure probability $\delta' \in (0, 1/e)$ are also specified. 
First, we set $\delta = \frac{\delta'}{10LT^2 \log_2 T}$. 
The meta-algorithm tracks a candidate algorithm index $\hat \imath_t$, corresponding to pair $(\AA_{\hat \imath_t}, \FF_{\hat \imath_t})$ that is believed to be $\RR_{\hat \imath_t}$-compatible at any given time --- as well as a set $B_t$ of indices of algorithms with more complex models. 
At the start of each episode, the meta-algorithm determines whether to use the algorithm $\AA_{\hat \imath_t}$ or explore using a randomly selected algorithm from the indices $B_t$, based on the outcome of a Bernoulli variable $U_t$ with success probability $1/{t^\kappa}$ where $\kappa \in (0, 1/2]$. 
This random variable $U_t$ represents an indicator that model exploration will occur.
After executing the policy from the chosen algorithm, the data is fed back to the algorithm to update, and a test is run to determine whether the algorithm should reject $\AA_{\hat \imath_t}$. The test checks the following condition for each $j \in B_t$:
\eq{
\GG_t(\hat \imath_t, j) > \WW(  | \Tau_{t}^{\hat\imath_t}|, \RR_{\hat \imath_t}, d_{\hat \imath_t}, \delta )
}
where for all $i < j \in [L]$, $t \in [T]$, $\Tau^i_t$ is the set of times when $\AA_i$ is chosen up to $t$, and $\GG$ is a scaled estimate of the excess gap between models $i$ and $j$, given by 
\eq{
\GG_t(i, j)  :=  \frac{| \Tau_{t}^{i}|}{|\Tau_{t}^j|} \sum_{t' \in \Tau_{t}^j}  g_{t'} - \sum_{t' \in \Tau_{t}^{i}} g_{t'} 
}
and $\WW$ is defined as
\eq{
\WW(t, \RR, d, \delta)  & := C_{\WW} \cdot   \RR(d, H, \log(T/\delta)) \cdot \sqrt{t}  \\
& \quad +  C_\WW \cdot  H \sqrt{ L t^{1 + \kappa} \cdot \log (1/\delta)}  \\
& \quad  +  C_\WW \cdot H \sqrt{  t \cdot \log (1/\delta)   } 
}
for a sufficiently large constant $C_{\WW} > 0$. The test is only valid after a minimal ``burn-in" period, $t \geq \tau_{\min}(\delta) = C_{\min} \cdot L^{\frac{2}{1 - \kappa}}\log^{\frac{1}{1 - \kappa}}(1/\delta)$ for a sufficiently large $C_{\min} > 0$, so this condition is also checked. If these conditions are true for some $j \in B_t$, meaning that the test fails, then \ms{} rejects $\AA_{\hat \imath_t}$ and switches to $\AA_{\hat \imath_t + 1}$. This process repeats until episode $T$. 

Note that although the algorithm uniformly explores among the algorithms in $B_t$, it does not require any explicit uniform or directed exploration within episodes that may be a tougher problem in RL settings than regret-minimization---one can simply run the algorithms as they were prescribed. 
In fact, we can interpret our meta-algorithm as automatically leveraging the exploration already in-built in the regret-minimizing base algorithms.

\begin{figure}
	\begin{algorithm}[H]
		\caption{ Explore-Commit-Eliminate (\ms{}) } \label{alg::ms}
		\begin{algorithmic}[1]
			\STATE \textbf{Input}: $\{ \AA_i, \FF_i,\RR_i, d_i \}_{i \in [L]}, L, T, \delta',\tau_{\min}(\cdot)$
			\STATE $\delta \leftarrow \frac{\delta'}{10LT^2 \log_2 T}$,  $\hat \imath_t \leftarrow 1$,  $\Tau^i_1 = \emptyset$ for all $i \in [L]$, $B_1 = [2, L]$.
			\STATE $U_t = \begin{cases}
										0 & \text{w.p. } 1 - \frac{1}{t^\kappa} \\
					1 & \text{w.p. } \frac{1}{t^\kappa} 
			\end{cases}$ for all $t \in [T]$.
			\FOR{ $t = 1, \ldots, T$ }
		
            \STATE Set $j =\begin{cases}
                \hat \imath_t & U_t = 0 \\
                J_t \sim \unif\{ B_t\}  & U_t = 1
            \end{cases}$
				\STATE $\Tau^{j}_t \leftarrow  \Tau^{j}_t \cup \{ t\}$ and $\Tau^{k}_t \leftarrow \Tau^{k}_t$ for all $k \neq j$.		
				\STATE Rollout policy $\pi_t$ from $\AA_{j}$
				\STATE Observe $z_t := (s_{t, 1}, u_{t, 1}, \ldots, u_{t, H}, s_{t, H + 1})$ and $g_t := \sum_{h \in [H]} r_{t, h}$
				\STATE Update $\AA_{j}$ with $t, z_t, g_t$
				\IF{ $t \geq \tau_{\min}(\delta)$ and there exists $j \in B_t$ such that $\GG_t(\hat \imath_t, j) > \WW(  | \Tau_{t}^{\hat\imath_t}|, \RR_{\hat \imath_t}, d_{\hat \imath_t}, \delta )$}
					\STATE $\hat \imath_{t + 1} \leftarrow \hat \imath_t + 1$
					\STATE $B_{t + 1} \leftarrow B_{t} \setminus \{ \hat \imath_{t+1} \}$
					\STATE If $\hat \imath_{t + 1} = L$, break and run $\AA_L$ to end of time
				\ELSE 
					\STATE $B_{t + 1} \leftarrow  B_t$
				\ENDIF
\ENDFOR
		\end{algorithmic}
	\end{algorithm}
\vspace{-.5cm}
\end{figure}

\subsection{Statistical test on excess gap}\label{sec::test}

The ability of \ms{} to judiciously accept or reject base algorithms lies in the simple statistical test at the end of each episode. The test can be viewed as a comparison between the scaled expected return obtained by a ``higher-order" algorithm, $\AA_j$, corresponding to index $j \in B_t$  during exploration rounds; and that of the active candidate algorithm $\AA_{\hat \imath_t}$ during all rounds of its usage.
If we find that the return of $\AA_j$ is significantly higher than that of $\AA_{\hat \imath_t}$, it suggests that switching to the more complex algorithm $\AA_j$ would yield significantly higher return, despite the fact that $\AA_{j}$ has a larger nominal regret bound and might have received much less data than $\AA_{\hat \imath_t}$ (as it is also competing for data with the other algorithms in $B_t$). 
The requirement that $t \geq \tau_{\min}(\delta)$ and our special choice of exploration schedule ensures that the algorithms in $B_t$ will have sufficient data to be useful in the test with high probability, while still exploiting the candidate model $\AA_{\hat \imath_t}$ whenever possible.

While we want \ms{} to reject lower-order models when they perform poorly, the test cannot be too sensitive. 
Otherwise, it could reject the optimal $i_*$ and choose some unnecessarily large $j> i_*$, leading to highly suboptimal model complexity dependence in the regret bound.
Our statistical test is designed to avoid this situation, as we prove in Section~\ref{sec::result}.

To give some additional intuition behind the test, it is useful to view the expected returns $\frac{1}{|\Tau_{t}^j|} \sum_{s \in \Tau_{t}^j}  g_s$ as a noisy \emph{lower bound} of the optimal value $V^*$; meanwhile the expected returns of $\frac{1}{|\Tau_{t}^{i_*}|} \sum_{s \in \Tau_{t}^{i_*}}  g_s$ plus the regret incurred, $\regret(\AA_{i_*})$, should be an \emph{upper bound} of the optimal value $V^*$ up to some noise as well, if $(\AA_{i_*}, \FF_{i_*})$ is $\RR_{i_*}$-compatible. Thus, as long as these intervals intersect, the test should succeed and $i_*$ continues to be accepted. If the intervals separate, the current candidate is rejected. 
This intuition is reflected in the three terms comprising the definition of $\WW$. The first is the nominal regret one expects to see from $\AA_{\hat \imath_t}$ if it is compatible. The last two follow from concentration of the averaging over returns of  the algorithms.

\section{MAIN RESULT}\label{sec::result}

Our main result shows that the meta-algorithm automatically adapts to the regret of the optimal pair $(\AA_{i_*}, \FF_{i_*})$ that is $\RR_{i_*}$-compatible. One of the main mechanisms behind this result is ensuring the validity of the test. The following lemma shows that \ms{} will never reject $(\AA_{i_*}, \FF_{i_*})$ with high probability. 

\begin{restatable}{lemma}{testLemma}
\label{thm::test}
We have $\GG_{t} (i_*, j ) \leq \WW(  | \Tau_{t}^{i_*}|, \RR_{i_*}, d_{i_*}, \delta )$ with probability at least $1 - \delta'$ for all $j \in [i_* + 1, L]$ and $t \geq \tau_{\min}(\delta'/ 10LT^2 \log_2 T )$.
\end{restatable}

Thus, since the meta-algorithm steps through the base-algorithms incrementally, Lemma~\ref{thm::test} shows that once it reaches $(\AA_{i_*}, \FF_{i_*})$, the first $\RR_{i_*}$-compatible pair, an algorithm with a more complex model class will never be selected. Our main theorem combines this result with the fact that, if the \ms{} has not rejected a misspecified algorithm $(\AA_{j}, \FF_{j} )$ with $j < i_*$, then the suboptimality of $\AA_j$ must not be significant. 

\begin{restatable}{theorem}{mainTheorem}
\label{thm::main}
Let the model exploration parameter $\kappa~=~1/3$.
Then, the model selection algorithm $\ms$ satisfies the regret bound
\eq{
& \widetilde{O} \left( HLT^{2/3} +  \RR_{i_*}( d_{i_*} , H, \log(LT/\delta'))   \cdot {i_*^{1/3} L^{1/2}}  T^{2/3}  \right) .
}
with probability at least $1 - \delta'$, where $\widetilde O$ hides logs and terms independent of $T$ and $\RR$.
\end{restatable}
The regret bound of the meta-algorithm matches that of the optimal algorithm in dependence on the complexity of its model class $d_{i_*}$ and horizon $H$, i.e., the best dependence if the optimal algorithm were provided \textit{a priori}.  We do incur a worse dependence on $T$, which is now $T^{2/3}$, compared to the nominal $\sqrt{T}$ rate, and a dependence of $L^{1/2}$, total number of algorithms, and $i_*$, the index of the optimal algorithm. Note that this type of trade-off in the parameter optimality for model selection is typical in recent results focused on contextual bandits, where methods making less strong assumptions typically incur sub-optimality in either the dependence on $d_{i_*}$ or $T$. In particular, Theorem~\ref{thm::main} matches the rate of Exp3.P~\citep{pacchiano2020model} and does so without non-trivially modifying the base algorithms. In addition to the minimax guarantee of Theorem~\ref{thm::main}, we show in Section~\ref{sec::id} that this can be improved to instance-dependent bounds, in contrast to Exp3.P and Corral.

\begin{table}[]
\small
\begin{tabular}{|l|l|l|}
\hline
\multicolumn{1}{|c|}{\textbf{Alg.}} & \multicolumn{1}{c|}{\textbf{Env.}} & \multicolumn{1}{c|}{\textbf{Regret}}
\\ \hline \hline
ModCB                                & CB                                    &     $\widetilde O\left(d^{1/3}_{i_*} T^{2/3}\right)$                                                                                            \\ \hline
OSOM                             & CB                                &              $\widetilde O\left(d^{1/2}_{i_*}T^{1/2}\right)$                                                                                   \\ \hline
Corral                               &       RL                         &           $\widetilde O \left( \RR_{i_*}^2 T^{1/2} \right)$                                                                                      \\ \hline
EXP3.P                             &     RL                           &                         $\widetilde O\left(\RR_{i_*} T^{2/3}\right)$                                                                        \\ \hline
Ours                             &         RL                       &       \begin{tabular}[c]{@{}l@{}} MM: $\widetilde O\left(\RR_{i_*} T^{2/3}\right)$ \\ ID: $\widetilde O \left(\RR_{i_*}^3 \Delta_{\min}^{-2} + \RR_{i_*}T^{1/2} + T^{2/3} \right)$\end{tabular}                                                                                                                   \\ \hline
\end{tabular}

\caption{We compare the theoretical guarantees of our algorithm to recent model selection work: ModCB \citep{foster2019model}, OSOM \citep{chatterji2020osom}, Corral \citep{agarwal2016corralling,pacchiano2020model}, and Exp3.P \citep{pacchiano2020model}. The first two apply to the contextual bandit (CB) setting and leverage distribution assumptions on the contexts to get nearly optimal regret. Corral and Exp3.P apply generally, but are suboptimal and require modifying the base algorithms in non-trivial ways. Our rate matches that of EXP3.P in the minimax (MM) setting without significant assumptions or modifications to the algorithms. We also achieve an improved instance-dependent (ID) rate when the gaps in performance between base algorithms are constant with minimal gap $\Delta_{\min}$.}
\vspace{-.5cm}
\end{table}\label{table:results}

\subsection{Proofs}

All proofs of Theorem~\ref{thm::main}, when not provided here, are available in Appendix~\ref{sec::omitted}. Due to space limitations, in this section, we prove Lemma~\ref{thm::test} and provide a proof sketch for Theorem~\ref{thm::main} to illustrate the main idea behind handling pairs $(\AA_{j}, \FF_{j})$ that are not $\RR_j$-compatible. In both cases, we require that three events hold and will show that they do with high probability. Define $\epsilon_t = g_t - V^{\pi_t}$ and let $\tau_i$ denote the first episode in which $\AA_{i}$ is chosen as the candidate $\hat \imath_t$.
If $\AA_{i}$ is never chosen then default to $\tau_i = T$. Recall that $\delta = \frac{\delta'}{10LT^2\log_2T}$.
\elist{
	\item Event $E_1$: For all $j \in [L]$ and all $t \in [T]$ such that $t\geq \tau_{\min}(\delta)$, if $t \leq \tau_i$, then $
\frac{t^{1 - \kappa}}{8L} \leq \abs{\Tau^i_t}  \leq 4t^{1 - \kappa}
$.
If $t > \tau_i$, then $\abs{\Tau^i_t} \leq t - \tau_i + 4 t^{1 - \kappa}$
\item Event $E_2$: For all $t \in [T]$,
\eq{
\textstyle
 \sum_{t' \in \Tau^{i_*}_{t}} V^* - V^{\pi_{t'}}  
 \leq \RR_{i_*}(d_{i_*}, H, \log(T/\delta))  \sqrt{| \Tau_{t}^{i_*} | } 
}
\item Event $E_3$: For all $j \in [L]$ and all $t \in [T]$,
$
\abs{ \sum_{t' \in \Tau_t^j } \epsilon_{t'} }  \leq H \sqrt{2 | \Tau_{t}^j| \log (2/\delta) }
$
}
The first event ensures that the exploration schedule yields sufficient data to all the algorithms before they are chosen. The second states that the nominal anytime regret guarantee holds for $(\AA_{i_*}, \FF_{i_*})$. The third handles concentration of the noisy returns that the algorithm observes from deploying policies.
The following lemma shows that all three events happen with high probability.

\begin{restatable}{lemma}{eventLemma}
\label{lem::event}
The event $E = \bigcap_{i \in \{1, 2,3\}} E_i$ holds with probability at least $1 - 10LT^2 \delta \log_2 T$.
\end{restatable}

Lemma~\ref{lem::event} is proved in Appendix~\ref{sec::eventlemmaproof}.
The proof for the first event uses a Freedman inequality (details in Appendix~\ref{sec::freedman}) to bound the sizes of all sets given that enough time has passed. The second event holds with high probability under the assumption that $(\AA_{i_*}, \FF_{i_*})$ is $\RR_{i_*}$-compatible. The third event can be shown to hold with high probability using the Azuma-Hoeffding inequality with appropriate union bounds.

\subsubsection{Proof of Lemma~\ref{thm::test}}

We now prove the statement of Lemma~\ref{thm::test} under the event $E$.
Adding and subtracting the sum of appropriately scaled value functions $\sum_{t' \in \Tau^j_t} V^{\pi_{t'}}$ and $\sum_{t' \in \Tau^{i_*}_t} V^{\pi_{t'}}$, we can write $\GG_t(i_*, j)$ in terms of value functions and conditionally zero-mean errors:
\eq{
 \GG_t(i_*, j) & =  \frac{\abs{\Tau_t^{i_*}}}{\abs{\Tau_t^{j}}} \sum_{ t' \in \Tau^{j}_t} g_{t'} - \sum_{ t' \in \Tau^{{i_*}}_t} g_{t'} \\ 
\!\!\!\!\!\!& = \frac{\abs{\Tau_t^{i_*}}}{\abs{\Tau_t^{j}}} \sum_{ t' \in \Tau^{j}_t}  \left( V^{\pi_{t'}} + \epsilon_{t'} \right) - \sum_{ t' \in \Tau^{i_*}_t}  \left( V^{\pi_{t'}}  + \epsilon_{t'}\right) 
\\
 & \leq  \sum_{ t' \in \Tau^{i_*}_t}  \left( V^* - V^{\pi_{t'}} \right)  +  \frac{\abs{\Tau_t^{i_*}}}{\abs{\Tau_t^{j}}} \sum_{ t' \in \Tau^{j}_t} \epsilon_{t'} - \sum_{ t' \in \Tau^{i_*}_t} \epsilon_{t'} 
 } %
 The last inequality follows as $V^* \geq V^{\pi_{t'}}$ for all $t' \in [T]$. If events $E_2$ and $E_3$ hold then 
 \eq{
 \GG_t(i_*, j) & \leq  \RR_{i_*}\left(d_{i_*}, H, \log(1 / \delta ) \right) \cdot  \sqrt{| \Tau_{t}^{i_*} | }  \\
 & \quad + H \sqrt{ 2 \abs{ \Tau_{t}^{i_*} } \log(2/\delta)  } + H \sqrt{  \frac{2 \abs{\Tau_t^{i_*}}^2}{\abs{\Tau_t^{j}}} \log( 2/\delta)    }  
}
By event $E_1$ and the fact that $j > i_*$ and $t \geq \tau_{\min}(\delta)$, $| \Tau^j_t | \geq \frac{t^{1 -\kappa }}{8L} \geq \frac{\abs{ \Tau^{i_*}_t }^{1 -\kappa }}{8L}$. Therefore, for the third term,
\eq{
H \sqrt{  \frac{2 \abs{\Tau_t^{i_*}}^2}{\abs{\Tau_t^{j}}} \log( 2/\delta)  }  & \leq H \sqrt{  16 L  \abs{\Tau_t^{i_*}}^{1 + \kappa} \log( 2/\delta)  } 
}
Applying this bound to the result in the previous display and given the definition of $\WW$, it follows that $\GG_t(i_*, j) \leq \WW(  \abs{\Tau_t^{i_*}}, \RR_{i_*}, d_{i_*}, \delta)$ for a sufficiently large constant $C_{\WW} > 0$, independent of $t$, $d_{i_*}$, $H$, and $\delta$.

\subsubsection{Proof Sketch of Theorem~\ref{thm::main}}

In bounding the regret of the meta-algorithm, there are three cases to handle: (1) before the test becomes valid, (2) once the test is valid but $i_*$ has not been chosen yet, and finally (3) once $i_*$ is chosen. We address the first and third cases before addressing the second, which is more involved. We define $\tau_* = \tau_{i_*}$ for shorthand.

Case (1): When $t < \tau_{\min}(\delta)$, the test to determine switching among any of the model classes is not yet valid. Here we simply pay the burn-in period giving
$
\regret_{1:\tau_{\min}(\delta) - 1} \leq O (H L^{\frac{2}{1 - \kappa}} \log^{\frac{1}{1 - \kappa}} (1/\delta) )
$.

 Case (3): If $t > \tau_*$, then the meta-algorithm has switched to $\AA_{i_*}$. Under event $E$, the condition in Lemma~\ref{thm::test} is met and so the test no longer fails. Therefore  $(\AA_{i_*}, \FF_{i_*})$ which is $\RR_{i_*}$-compatible is not  rejected in the remaining episodes. The regret during this phase scales as $\RR_{i_*}(d_{i_*}, H, \log(T/ \delta)) \cdot \sqrt{T}$ plus additional $O(HLT^{1 - \kappa})$ regret due to exploration of the remaining base algorithms in $B_t$. %

 Case (2) is when $\tau_{\min} < t \leq \tau_{*}$---the test is eligible but the meta-algorithm is either switching among misspecified models or unable to detect that they are misspecified. Since the misspecification is not detected for any of the algorithms in $B_t$, we know 
$
\GG_t(\hat \imath_t, i_*) \leq \WW(  | \Tau_{t}^{\hat \imath_t}|, \RR_{\hat \imath_t}, d_{\hat \imath_t}, \delta ).
$
That is, the average reward for $\AA_{\hat \imath_t}$ is not significantly different from that of $\AA_{i_*}$. Since $\AA_{i_*}$ is only played during exploration and $t \geq \tau_{\min}(\delta)$, its number of rounds played can be lower bounded by $t^{1 - \kappa}/8L$ and thus its average regret is at most roughly \eq{
\widetilde O \left( \frac{ L^{1/2}\RR_{i_*}(d_{i_*}, H, \log(T/\delta)) }{ t^{\frac{1 - \kappa }{2} }}\right).
}

The success of the test suggests that the average reward of $\AA_{\hat \imath_t}$ should be close to this. Extrapolating over the rounds played by $\AA_{\hat \imath_t}$, the regret for $\hat \imath_t$ will be \eq{
\widetilde O \left( \RR_{i_*}(d_{i_*}, H, \log(T/\delta)) \cdot L^{1/2} \abs{\Tau^{\hat \imath_t}_t } ^{\frac{1 + \kappa }{2}} \right)
}
up to a constant shift by $\WW(  | \Tau_{t}^{\hat \imath}|, \RR_{\hat \imath_t}, d_{\hat \imath_t}, \delta )$. The shift is dominated by the above display because $\RR_{\hat \imath_t} \leq \RR_{i_*}$ and $\kappa \in (0, 1/2]$. Finally, since we must account for the cumulative effect for all $i < i_*$, Jensen's inequality shows the sum of these terms is bounded above by \eq{
\widetilde O \left( \RR_{i_*}(d_{i_*}, H, \log(T/\delta)) \cdot i_*^{\frac{1 - \kappa }{2}}L^{1/2} T ^{\frac{1 + \kappa }{2}}  \right) .
}
This becomes the dominant term in the regret. Additional regret of $O(HLT^{1-  \kappa} + Hi_* + H T^{\frac{1 + \kappa}{2}} \log^{1/2}(1/\delta) )$ is also paid for exploration, switching costs, and estimation error of the averages.
Summing these three cases and taking $\kappa = 1/3$ proves Theorem~\ref{thm::main}.

\subsection{Applications}
Though Theorem~\ref{thm::main} is stated generally for any RL algorithms with nominal anytime regret bounds, we can easily specialize it to several important problem settings without knowing the optimal model class \textit{a priori}. Formal details can be found in Appendix~\ref{sec::applications}.

\paragraph{Linear Models} Recent work has considered linear MDPs where the transition dynamics and reward are linear in some feature vector~\citep{jin2019provably, yang2019reinforcement}. We assume access to nested features $\phi_{i} : \SSS \times \AA \to \R^{d_i}$ for $i \in [L]$ such that $d_{i} \leq d_{i + 1}$ and the first $d_i$ components of $\phi_{i + 1}$ are the same as $\phi_i$. These feature generate linear model classes:
\begin{align}\label{eq::linear-model}
\FF_i = \left\{ (s, a) \mapsto \< \phi_i(s, a), \theta\> \ : \ \theta \in \R^{d_i} \right\} 
\end{align}
$\FF_i$ realizes $\MM$ if it has zero approximation error for the transition dynamics $P(\cdot | s, a)$ and reward $r(s, a)$. Let $i_*$ be the smallest index such that $\FF_{i_*}$ realizes $\MM$. The regret of LSVI-UCB \citep{jin2019provably} under $\FF_{i}$ for $i \geq i_*$ is $\widetilde{ O} \left( \sqrt{d_{i}^3 H^4 T }  \right)$. Using \ms{} with LSVI-UCB algorithms guarantees $
\regret_T(\ms{}) = \widetilde O \left( \sqrt{d_{i_*}^3 H^4 } \cdot L^{5/6}T^{2/3}\right)$.
MatrixRL \citep{yang2019reinforcement} similarly assumes a linear function class: \eq{\FF_{i} = \left\{(s, u, s') \mapsto \phi_i(s,u)^\top M \psi_i(s')  \ : \ M \in \R^{d_i \times d_i'}\right\}
} for $\psi_i(s') \in \R^{d_i'}$. For $\FF_i$ with $i \geq i_*$ that realizes the transition dynamics $P$, the regret is $\widetilde O \left( \sqrt{d_i^{3} H^5 T }\right)$. Our model selection algorithm, \ms{}, achieves $\regret_T(\ms) =\widetilde O \left( \sqrt{d_{i_*}^{3} H^5 } \cdot L^{5/6} T^{2/3}\right)$.

A more general linear setting considers learning under low Bellman error without directly assuming linearity of $P$. With $\FF_i$ defined as in (\ref{eq::linear-model}), we say  $\FF_i$ realizes $\MM$ if it has zero inherent Bellman error (Definition 1, \cite{zanette2020learning}). Then for $i \geq i_*$, ELEANOR~\citep{zanette2020learning} guarantees improved regret $\widetilde O \left( d_i\sqrt{H^4 T} \right)$ and \ms{} achieves $\regret_T(\ms{}) = \widetilde O \left( d_{i_*} \sqrt{H^4} \cdot L^{5/6} T^{2/3} \right)$.

As done by \cite{foster2019model}, for nested model classes, the $L$ dependence can be replaced by $\log T$ by only considering a subset of features such that $d_i = O(2^i)$ for $i \in [\ceil{\log_2(T)}]$.

\paragraph{Low Bellman Rank} For more general function approximation, consider the setting of MDPs with low Bellman rank~\citep{jiang2017contextual} and finite (but not necessarily linear or nested) models $\{ \FF_i\}$ with $\FF_i : \SSS \times \UU \to \R$. $\FF_i$ realizes $\MM$ if there is $f^* \in \FF_i$ such that $f^* = Q^*_h$ for all $h \in [H]$ and the induced Bellman rank is $M_i < |\FF_i|$. For $i$ such that $\FF_i$ realizes $\MM$, AVE~\citep{dong2020n} guarantees regret $\widetilde O \left(\sqrt{ M_i^2 |\UU | H^4 T \log^3 |\FF_i| } \right)$. Let $i_*$ be defined similarly as before. 
Then \ms{} achieves $\regret_T(\ms) = \widetilde O \left(\sqrt{ M_{i_*}^2 |\UU | H^4  \log^3 |\FF_{i_*}| } \cdot L^{5/6} T^{2/3} \right)$

\section{INSTANCE-DEPENDENT BOUNDS}\label{sec::id}
We now prove a stronger ``instance-dependent" guarantee on online selection over more specialized base algorithms which have provable regret guarantees that are sublinear in $T$, but compared to the best policy within its respective policy class. For example, for an algorithm and model class $(\AA, \FF)$ using value-based function approximation we might consider the greedy policy class:
\eq{
\Pi_{\FF} = \left\{ (s, h) \mapsto \argmax_{u \in \UU} \ f(s, u, h) \ : \   f \in \FF \right\}.  
}
The regret with respect to the best-in-class is 
\eq{
\textstyle
\regret_{T}(\AA, \Pi_{\FF}; \MM) = \max_{\pi \in \Pi_{\FF}} \sum_{t \in [T]} V^{\pi} - V^{\pi_t}
}
To consider algorithms that may obtain sublinear regret with respect to this weaker benchmark but not with respect to $V^*$, we give a refined definition of $\RR$-compatible algorithms.
\begin{definition}\label{def::realize2}
The pair $(\AA, \FF)$ is said to be $\RR^{\Pi_{\FF}}$-compatible with respect to $\Pi_{\FF}$ 
on the MDP $\MM$ if we have
\eq{
\regret_{T}(\AA, \Pi_{\FF}; \MM)  \leq \RR^{\Pi_{\FF}}(d_{\FF}, H, \log(T/\delta)) \cdot \sqrt{t}
}
for all $t$ with probability at least $1 - \delta$.
\end{definition}

The value of $\max_{\pi \in \Pi_{\FF}} V^\pi$ is typically unknown because of the complex dependence between $\Pi_{\FF}$ and $\MM$, and because $\Pi_{\FF}$ is often determined by $\FF$. Given a set of algorithms with different policy classes, we would like to select the one with the smallest regret compared to the optimal best-in-class value. Formally, we assume there are given algorithms $ \{ (\AA_i, \FF_i) \}$ with policy classes $\{ \Pi_i\}$ each having optimal values $V_i^* := \max_{\pi \in \Pi_i} V^{\pi}$ and regret coefficients $\{ \RR_i^{\Pi_i}\}$ such that \textit{for all
}$i$ the pair $(\AA_i, \FF_i)$ is $\RR_i^{\Pi_i}$-compatible and $\RR_i(d_i, \cdot, \cdot)  \leq \RR_{i + 1} (d_{ i +1}, \cdot, \cdot)$. 
Our goal is to select $i_* \in B_* :=\argmax_{j \in [L]} V_j^*$
that has the smallest complexity dependence i.e. $i_* = \argmin_{i \in B_*} \RR_i^{\Pi_{i}}(d_i, \cdot, \cdot) $. We emphasize that even if no algorithm is compatible in the sense of Definition~\ref{def::realize}, we want the optimal best-in-class guarantee\footnote{In essence, the best-in-class guarantee needs to hold even under model misspecification. A good example of a base algorithm satisfying this condition would be Exp4 in the contextual bandits setting.} in the sense of Definition~\ref{def::realize2}.

The difference between this setting and the last is that all algorithms are assumed to be compatible with respect to their own policy classes now, but the differing $\Pi_i$ mean that some can have lower $V^*_i$, which we want to eliminate.
Note that although the regret coefficients are ordered as in (\ref{eq::ordering}), the values $\{ V_i^*\}$ are unknown and not necessarily ordered. Observe that $i_* = \min B_*$, so that $V_{i_*}^* > V_i^*$ for all $i < i_*$ and $V_{i_*}^* \geq V_{i}^*$ for all $i > i_*$. Thus $i_*$ has the lowest regret for the best policy class. We would like an algorithm $\AA$ that bounds $\regret_{T}(\AA, \Pi_{i_*}; \MM)$ with dependence on only the complexity of $\FF_{i_*}$. 
The following result shows that Algorithm~\ref{alg::ms}, without any modifications, can obtain an \textit{instance-dependent} regret guarantee based on the size of the gaps $\Delta_{j, i_*}:= V_{i_*}^* - V_j^*$ for $j < i_*$.

\begin{restatable}{theorem}{idTheorem}
\label{thm::id}
For a given $\MM$, let $(\AA_i ,\FF_i)$ be $\RR_i^{\Pi_i}$-compatible with respect to $\Pi_i$ for all $i \in [L]$. Then, with probability at least $1 - \delta'$, \ms{} with $\kappa = 1/3$  satisfies the regret bound with respect to policy class $\Pi_{i_*}$:
\eq{
\textstyle
\widetilde O \left(H L T^{2/3} + \RR_{i_*}^{\Pi_{i_*}} \sqrt{T}  +
L^{3/2}(\RR_{i_*}^{\Pi_{i_*}})^3\sum_{ i < i_*} \Delta_{i, i_*}^{-2} \right)
}
If $\kappa = 1/2$, then it satisfies
\eq{
\textstyle
\widetilde O \left(H L\sqrt{T} + \RR_{i_*}^{\Pi_{i_*}} \sqrt{T}  +
L^{2}(\RR_{i_*}^{\Pi_{i_*}})^4\sum_{ i < i_*} \Delta_{i, i_*}^{-3} \right)
}
\end{restatable}

Comparing this result to Theorem~\ref{thm::main}, if \ms{} is run with the same $\kappa = 1/3$ and the gaps are constant, a significantly better rate is possible since the third term has no dependence on $T$. With a more aggressive exploration choice of $\kappa = 1/2$, an even stronger instance-dependent guarantee is possible, matching the optimal $\RR_{i_*}^{\Pi_{i_*}}\sqrt{T}$ rate of the best algorithm. However, this comes at the price of worse dependence on the gaps and $\RR_{i_*}^{\Pi_{i_*}}$ factors, in the term that does not increase polynomially with $T$. In either case, Theorem~\ref{thm::id} shows that we can obtain optimal or near-optimal dependence in $T$ and only suboptimal $\RR_{i_*}^{\Pi_{i_*}}$-dependence on terms that do not grow with $T$, as long as the gaps are constant. In Appendix~\ref{sec::v-star}, we show that these rates can be even further improved with only minimal modifications to \ms{} if given access to fast estimators of the gaps or $V^*$.

\section{CONCLUSION}

We present a new model selection meta-algorithm for RL with function approximation. Given a set of base algorithms in which one is well-specified, the meta-algorithm adapts to the regret of the optimal one using a simple and interpretable statistical test. The regret of the meta-algorithm retains optimal dependence on model complexity while increasing the dependence on the number of episodes, $T$, to $O(T^{2/3})$. Compared to past efforts, our meta-algorithm provides similarly strong worst-case regret bounds, is computationally efficient conditioned on efficiency of the base algorithms, works under minimal assumptions, and provides new instance-dependent results.

Of many interesting future directions, a particularly interesting one given the prior significance of access to $V^*$ \citep{foster2019model,modi2020sample} and our even stronger instance-dependent regret rates (see Appendix~\ref{sec::v-star}), is whether estimating $V^*$ is easier than estimating the optimal policy.

\subsubsection*{Acknowledgements}
JNL is supported by the NSF GRFP.
VM acknowledges support from a Simons-Berkeley Research
Fellowship.
Part of this work was done while some of the authors were visiting the Simons Institute for the Theory of Computing.

\bibliographystyle{plainnat}
\bibliography{refs}
\raggedbottom
\pagebreak

\appendix
\onecolumn

\section{OMITTED PROOFS}\label{sec::omitted}

In this section, we collect proofs for Theorem~\ref{thm::main} that were omitted from the main paper.

\subsection{Proof of Lemma~\ref{lem::event}}\label{sec::eventlemmaproof}

Here, we restate and prove Lemma~\ref{lem::event}.
\eventLemma*
\begin{proof}
We will show that each of the three events holds with high probability and the apply the union bound.

Corollary~\ref{cor::set-sizes} of Section~\ref{sec::freedman} shows event $E_1$ holds with probability at least $1 - 4 LT^2 \delta \log_2 T$.

For event $E_2$, $i_*$ is the index of the algorithm that is $\RR_{i_*}$-compatible and anytime. Let $\pi_{(k)}^{i_*}$ denote the policy played by $\AA_{i_*}$ at the $k^{th}$ call to $i_*$. For $K \in [T]$, these properties guarantee its regret bound holds, with probability at least $1 - \delta$, \eq{
\sum_{k \in [K]} V^* - V^{\pi_{(k)}^{i_*}} \leq   \RR_{i_*}( d_{i_*}, H, \log(T/\delta))  \cdot \sqrt{K} 
}
Taking the union bound over all $K \in [T]$ shows that event $E_2$ holds with probability at least $1 - T \delta$.

As in the previous case, we can view the process $\epsilon^i_{(1)}, \ldots, \epsilon^i_{(T)}$ as the pre-drawn differences between the observed and expected returns for the $1$ through (at most) $T$ times of playing model $\AA_i$. Applying the Azuma-Hoeffding inequality with $|\epsilon^i_{(k)}| \leq H$ and taking the union bound over all $K \in [T]$,
\eq{
| \sum_{k \in [K]} \epsilon^i_{(k)} | & \leq H \sqrt{2 K \log (2/\delta) } 
}
with probability at least $1 - T\delta$. Taking the union bound over all models, event $E_3$ occurs with probability at least $1 - LT \delta$.

Taking these events together and $\delta' = 10LT^2 \delta \log_2 T$, event $E$ holds with probability at least $1 - \delta'$.
\end{proof}

\subsection{Full Proof of Theorem~\ref{thm::main}}

Here, we restate and complete the proof of Theorem~\ref{thm::main}.
\mainTheorem*
\begin{proof}
Let $\tau_* := \tau_{i_*}$ denote the time that $\AA_{i_*}$ is chosen as the candidate. Recall that $\delta = \frac{\delta'}{10LT^2 \log_2 T}$.
The analysis can be divided into three phases when conditioned on the event $E$.
\begin{enumerate}
\item $t < \tau_{\min}(\delta)$: the test to determine switching to $i_*$ is not valid yet.
\item $\tau_{\min}(\delta) < t \leq \tau_*$: the test is eligible but \ms{} is still switching among incompatible algorithms.
\item $t > \tau_*$: \ms{} has switched to $\AA_{i_*}$.
\end{enumerate}
Note that it is possible that $\tau_* \geq T$. That is, the algorithm only uses incompatible algorithms; however, we will show that this case still guarantees regret that adapts to the optimal algorithm $i_*$.

\paragraph{Case 1: Invalid Test} 
We require $t \geq \tau_{\min}(\delta)$ in order for the condition in Lemma~\ref{thm::test} to hold under $E$ when $\hat \imath_t = i_*$. Therefore, we can view this period $t < \tau_{\min}(\delta)$ as an unavoidable burn-in period. The regret during this interval can then be upper bounded in the worst case as
\eq{
\regret_{1:\tau_{\min}(\delta) - 1}  = \sum_{t = 1}^{\tau_{\min} - 1} V^* - V^{\pi_t}  \leq H \tau_{\min} = O \left( H L^{\frac{2}{1 - \kappa}} \log^{\frac{1}{1 - \kappa}}(1/\delta)\right)
}

\paragraph{Case 2: Misspecified Case} In the second phase, the test is valid, but \ms{} is either utilizing algorithms below $i_*$ or switching among them in the event the test fails.  The regret can be decomposed across each set $\Tau^j_{\tau_*}$ of times playing $\AA_j$ up to time $\tau_*$:
\eq{
\regret_{\tau_{\min}(\delta) :\tau_*} & = \sum_{j \in [L]} \sum_{t \in \Tau^j_{\tau_*}} V^* - V^{\pi_t} \\
& \leq   4H (L - i_*) \tau_*^{1- \kappa} +   \sum_{j < i_*} \sum_{t \in \Tau^j_{\tau_{j + 1}}} V^* - V^{\pi_t} \\
& \leq 4H (L - i_*) \tau_*^{1- \kappa} +  H i_* +  \sum_{j < i_*} \sum_{t \in \Tau^j_{\tau_{j + 1} - 1}} V^* - V^{\pi_t}
}
The second line follows from the fact that for $j > i_*$, algorithm $j$ is not selected yet (if ever), so maximal regret is paid for those algorithms during exploration. Event $E_1$ upper bounds the number of times that can be in $\Tau^j_{\tau_*}$ at time $\tau_*$, since the regret due to $j$ is only due to exploration. Furthermore, for $j< i_*$, once $j$ is rejected, it is never used for exploration again, so we can replace $\Tau^j_{\tau_*}$ with $\Tau^j_{\tau_{j + 1}}$ for $j < i_*$. The third line is necessary as no guarantee is given during episodes when a test fails and there can be at most $i_*$ failing tests since the condition in Lemma~\ref{thm::test} is always true under event $E$.

Then, we focus on bounding the right-hand term. Fix $j < i_*$. Observe that for $t \in \T^j_{\tau_{j + 1} - 1}$ the tests succeed for all comparisons including with $i_*$: 
\eq{
\GG_{\tau_{j + 1} - 1}(j, i) \leq \WW(\abs{\Tau^j_{\tau_{j + 1} - 1}  }, \RR_{j}, d_j, \delta)
}
for all $i > j$. Therefore, since $i_* > j$, the definition of $\GG$ can be used the bound the following:
\eq{
\sum_{t \in \T^j_{\tau_{j + 1} - 1}} V^* - V^{\pi_t} & = \sum_{ t \in \T^j_{\tau_{j + 1} - 1}} (V^* - g_t ) + \sum_{t \in \T^j_{\tau_{j + 1} - 1}} \epsilon_t \\
& \leq \frac{ \abs{ \T^j_{\tau_{j + 1} - 1} } }{ \abs{ \T^{i_*}_{\tau_{j + 1} - 1} }} \sum_{ t \in \T^{i_*}_{\tau_{j + 1} - 1}} (V^* - g_t ) + \WW(\abs{ \T^j_{\tau_{j + 1} - 1} }, \RR_j, d_j, \delta ) +   \sum_{t \in \T^j_{\tau_{j + 1} - 1}} \epsilon_t  \\
& \leq \frac{ \abs{ \T^j_{\tau_{j + 1} - 1} } }{ \abs{ \T^{i_*}_{\tau_{j + 1} - 1} }} \sum_{ t \in \T^{i_*}_{\tau_{j + 1} - 1}} (V^* - V^{\pi_t} ) + \WW(\abs{ \T^j_{\tau_{j + 1} - 1} }, \RR_j, d_j, \delta ) \\
& \quad  +   \sum_{t \in \T^j_{\tau_{j + 1} - 1}} \epsilon_t  + \frac{ \abs{ \T^j_{\tau_{j + 1} - 1} } }{ \abs{ \T^{i_*}_{\tau_{j + 1} - 1} }} \sum_{ t \in \T^{i_*}_{\tau_{j + 1} - 1}} \epsilon_t
}
Now we can use the fact that $E_2$ and $E_3$ hold to bound the regret and estimation errors:
\begin{align}\label{eq::thm1-regret-ub}
\begin{split}
\sum_{t \in \T^j_{\tau_{j + 1} - 1}} V^* - V^{\pi_t} & \leq  O \left( \RR_{i_*}(d_{i_*}, H, \log(T/\delta)) \cdot \sqrt{\frac{\abs{ \T^j_{\tau_{j + 1} - 1} }^2 } { \abs{ \T^{i_*}_{\tau_{j + 1} - 1} } } }  \right) + \WW(\abs{ \T^j_{\tau_{j + 1} - 1} }, \RR_j, d_j, \delta ) \\
& \quad + O \left( H \sqrt{ \abs{\T^j_{\tau_{j + 1} - 1}}  \cdot \log(1/\delta) }  \right) + O \left( H \sqrt{ \frac{ \abs{\T^j_{\tau_{j + 1} - 1}}^2 } { \abs{\T^{j}_{\tau_{j + 1} - 1}} }  \cdot \log(1/\delta) }  \right)
\end{split}
\end{align}
Using $E_1$ and the fact that $\tau_{\min}(\delta) \leq \tau_{j + 1} - 1 \leq \tau_*$, we have that \eq{\abs{ \Tau^{i_*}_{\tau_{j + 1} - 1} } \geq \frac{(\tau_{j + 1} - 1)^{1 - \kappa}}{8L} \geq \frac{\abs{ \Tau^{i_*}_{\tau_{j + 1} - 1} }^{1- \kappa}}{8L}.}
Then the terms in (\ref{eq::thm1-regret-ub}) that contain $\abs{\Tau^{i_*}_{\tau_{j + 1} - 1}}$ in the denominator can be upper bounded:
\eq{
O \left( \RR_{i_*}(d_{i_*}, H, \log(T/\delta)) \cdot \sqrt{\frac{\abs{ \T^j_{\tau_{j + 1} - 1} }^2 } { \abs{ \T^{i_*}_{\tau_{j + 1} - 1} } } }  \right) 
& \leq O \left( L^{1/2} \RR_{i_*}(d_{i_*}, H, \log(T/\delta)) \cdot \abs{ \T^j_{\tau_{j + 1} - 1} }^{\frac{1 + \kappa}{2}}    \right) \\
O \left( H \sqrt{ \frac{ \abs{\T^j_{\tau_{j + 1} - 1}}^2 } { \abs{\T^{j}_{\tau_{j + 1} - 1}} }  \cdot \log(1/\delta) }  \right)
& \leq O \left( H L^{1/2} \abs{\T^j_{\tau_{j + 1} - 1}}^{\frac{1 + \kappa}{2}}   \cdot \log^{1/2}(1/\delta)   \right)
}
The bound then becomes
\eq{
\sum_{t \in \T^j_{\tau_{j + 1} - 1}} V^* - V^{\pi_t} & \leq O \left( L^{1/2} \RR_{i_*}(d_{i_*}, H, \log(T/\delta)) \cdot \abs{ \T^j_{\tau_{j + 1} - 1} }^{\frac{1 + \kappa}{2}}    \right) + \WW(\abs{ \T^j_{\tau_{j + 1} - 1} }, \RR_j, d_j, \delta ) \\
& \quad + O \left( H  \abs{\T^j_{\tau_{j + 1} - 1}} ^{1/2} \cdot \log^{1/2}(1/\delta)   \right) + O \left( H L^{1/2} \abs{\T^j_{\tau_{j + 1} - 1}}^{\frac{1 + \kappa}{2}}   \cdot \log^{1/2}(1/\delta)   \right)
}
Since $\RR_j \leq \RR_{i_*}$, the regret for $j$ in this case is
\eq{
\sum_{t \in \T^j_{\tau_{j + 1} - 1}} V^* - V^{\pi_t} & \leq O \left( L^{1/2} \RR_{i_*}(d_{i_*}, H, \log(T/\delta)) \cdot \abs{ \T^j_{\tau_{j + 1} - 1} }^{\frac{1 + \kappa}{2}}  +  H L^{1/2} \abs{\T^j_{\tau_{j + 1} - 1}}^{\frac{1 + \kappa}{2}}   \cdot \log^{1/2}(1/\delta)  \right) 
}
Observe that $\sum_{j < i_*} \abs{ \Tau^j_{\tau_{j + 1} - 1 } } \leq T$ and the right-hand side is a sum of concave functions of each $\abs{ \Tau^j_{\tau_{j + 1} - 1 } }$. Using Jensen's inequality with the uniform distribution over $\abs{ \Tau^j_{\tau_{j + 1} - 1 } }$ for $j < i_*$ and then upper bounding by $T$ yields the bound:
\eq{
\regret_{\tau_{\min}(\delta) :\tau_*} & \leq O \left( HL T^{1 - \kappa} + H i_* +  \left( \RR_{i_*}(d_{i_*}, H, \log(T/\delta)) + H \log^{1/2}(1/\delta)  \right) \cdot i_*^{\frac{1 - \kappa}{2}} L^{1/2}  \cdot T^{\frac{1 + \kappa }{2}}   \right)
}

\paragraph{Case 3: Selecting $\AA_{i_*}$}
Starting at $\tau_* + 1$, $\AA_{i_*}$ is selected. Note that the condition in Lemma~\ref{thm::test} holds under event $E$, so \ms{} will never reject $i_*$. Then
\eq{
\regret_{\tau_* + 1 : T} & \leq \sum_{j \in [i_* + 1, L]} H |\Tau^j_T| + \sum_{t \in  \Tau^{i_*}_T} V^* - V^{\pi_t} \\
& \leq  \sum_{j \in [i_* + 1, L]} H |\Tau^j_T| + O  \left( \RR_{i_*}(d_{i_*}, H, \log(T/\delta) \cdot  \sqrt{T}  \right) \\
& \leq  O  \left( H L T^{1 - \kappa} +  \RR_{i_*}(d_{i_*}, H, \log(T/\delta) \cdot  \sqrt{T}  \right)
}
Adding the terms from these three phases gives the final bound:
\eq{
\regret_T = O \left(    
H L^{\frac{2}{1 - \kappa}} \log^{\frac{1}{1 - \kappa}}(1/\delta)
+ HL T^{1 - \kappa} + H i_* 
+ \left( \RR_{i_*}(d_{i_*}, H, \log(T/\delta)) + H \log^{1/2}(1/\delta)  \right) \cdot i_*^{\frac{1 - \kappa}{2}} L^{1/2}  \cdot T^{\frac{1 + \kappa }{2}}
\right) 
}
Then we choose $\kappa = 1/3$ to recover the statement in the theorem.
\end{proof}

\subsection{Proof of Theorem~\ref{thm::id}}

Here, we restate an prove Theorem~\ref{thm::id}.
\idTheorem*
\begin{proof}
First we will show that the sufficient events to prove this result occur with high probability. While the other events remain the same. we must modify  event $E_2$ from Lemma~\ref{lem::event} slightly because we are interested in the case when all algorithms are compatible with respect to their own policy classes. Let $E_2'$ denote the following event: for all $t \in [T]$ and $i \in [L]$,
\eq{
\textstyle
 \sum_{t' \in \Tau^{i}_{t}} V^*_i - V^{\pi_{t'}}  
 \leq \RR_{i}^{\Pi_i}(d_{i}, H, \log(T/\delta))  \sqrt{| \Tau_{t}^{i} | } 
}
As in Lemma~\ref{lem::event}, this almost follows from Definition~\ref{def::realize2}; however, we also union bound over all algorithms. Thus $E_2'$ occurs with probability at least $1 - LT\delta$. Let $E_1' = E_1$ and $E_3' = E_3$. Then $E' = \bigcap_{i \in {1, 2,3}} E_i'$ occurs with probability at least $1  - 10 LT^2\delta \log_2 T$, as before.

Recall that $i_* = \min B_*$ where $B_*$ is the set of indices that achieve maximal value, $\argmax_i V^*_i$. For shorthand, we will let $\RR_{j} := \RR_j^{\Pi_j} (d_j, H, \log(T/\delta))$.
We now verify that the statistical test will not fail once \ms{} reaches some $i_* \in B_*$. This is nearly identical to Lemma~\ref{thm::test}, but we must verify it with respect to values that are not the optimal value.
\begin{lemma}\label{lem::test2}
Let $(\AA_i, \FF_i)$ be an $\RR^{\Pi_i}_i$-compatible algorithm with respect to $\Pi_i$ for all $i \in [L]$ and let $i_* = \min B_*$.
Given that event $E'$ holds and $t \geq \tau_{\min}(\delta)$, then, for all $j \in [i_* + 1, L]$, it holds that $\GG_{t}(i_*, j) \leq \WW(\abs{\Tau_t^{i_*}}, \RR_{i_*}, d_{i_*}, \delta)$.
\end{lemma}
\begin{proof}
From the definition of $\GG$,
\eq{
 \GG_t(i_*, j) & =  \frac{\abs{\Tau_t^{i_*}}}{\abs{\Tau_t^{j}}} \sum_{ t' \in \Tau^{j}_t} g_{t'} - \sum_{ t' \in \Tau^{{i_*}}_t} g_{t'} \\ 
\!\!\!\!\!\!& = \frac{\abs{\Tau_t^{i_*}}}{\abs{\Tau_t^{j}}} \sum_{ t' \in \Tau^{j}_t}  \left( V^{\pi_{t'}} + \epsilon_{t'} \right) - \sum_{ t' \in \Tau^{i_*}_t}  \left( V^{\pi_{t'}}  + \epsilon_{t'}\right) 
\\
 & \leq  \sum_{ t' \in \Tau^{i_*}_t}  \left( V^*_{i_*} - V^{\pi_{t'}} \right)  +  \frac{\abs{\Tau_t^{i_*}}}{\abs{\Tau_t^{j}}} \sum_{ t' \in \Tau^{j}_t} \epsilon_{t'} - \sum_{ t' \in \Tau^{i_*}_t} \epsilon_{t'} 
 }
where the last step uses the fact that $V_{i_*}^* = \max_i V_i^*$. Since $(\AA_{i_*}, \FF_{i_*})$ is $\RR_{i_*}^{\Pi_{i_*}}$-compatible, the remainder of the proof is identical to that of Lemma~\ref{thm::test} by applying the conditions in $E'$.
\end{proof}

As before, in the full proof we handle three cases: (1) before the test is valid, (2) while $i < i_*$ is chosen, (3) after $i_*$ is chosen. In the first case, we again pay the burn-in period regret of $\regret_{1:\tau_{\min}(\delta) - 1} =     O(H\tau_{\min}(\delta))$. In the third, we showed that the test will never fail once $\hat\imath_t = i_*$. Therefore, $\regret_{\tau_*:T} =  O \left( HLT^{1-\kappa} + \RR_{i_*}\cdot \sqrt{T} \right)$.

To bound the regret during the misspecified phase, we construct an upper bound on the number of times $\AA_{j}$ can be played for $j < i_*$. Let $t$ be a time such that $\hat\imath_t = j < i_*$ and the test succeeds. First, we bound the size of the gaps.

Note that by definition $V_j^* \geq \frac{1}{\abs{\Tau_t^{j}}} \sum_{t' \in \Tau_t^{j}} V^{\pi_{t'}}$ and event $E'$ ensures that $V^*_{i_*} \leq \frac{\RR_{i_*}}{\abs{\Tau_t^{i_*}}^{1/2}} + \frac{1}{\abs{ \Tau^{i_*}_t }} \sum_{t' \in \Tau_t^{i_*}} V^{\pi_{t'}}$. Then,
\eq{
\Delta_{j, i_*} & = V_{i_*}^* - V_j^*  \\
    & \leq \frac{1}{\abs{ \Tau_t^{i_*} }} \sum_{t' \in \Tau_t^{i_*}} V^{\pi_{t'}} + \frac{\RR_{i_*}}{\abs{\Tau_t^{i_*}}^{1/2}} - \frac{1}{\abs{\Tau_t^{j}}} \sum_{t' \in \Tau_t^{j}} V^{\pi_{t'}} \\
    & = \frac{\RR_{i_*}}{\abs{\Tau_t^{i_*}}^{1/2}} + \frac{1}{\abs{ \Tau_t^{i_*} }} \sum_{t' \in \Tau_t^{i_*}} (g_{t'} - \epsilon_{t'})  - \frac{1}{\abs{\Tau_t^{j}}} \sum_{t' \in \Tau_t^{j}} (g_{t'} - \epsilon_{t'}) \\
    & \leq \frac{\WW (\abs{\Tau^{j}_t}, \RR_j, d_j, \delta)}{\abs{\Tau^j_t}} + \frac{\RR_{i_*}}{\abs{\Tau_t^{i_*}}^{1/2}} - \frac{1}{\abs{ \Tau_t^{i_*} }} \sum_{t' \in \Tau_t^{i_*}} \epsilon_{t'} + \frac{1}{\abs{\Tau_t^{j}}} \sum_{t' \in \Tau_t^{j}} \epsilon_{t'} \\
    & \leq C_{\WW} \cdot \left( \frac{\RR_{j}}{\abs{\Tau^j_t}^{1/2}} + H \sqrt{\frac{16 L \log(2/\delta) }{\abs{\Tau^{j}_t}^{1 - \kappa}}}  + H \sqrt{ \frac{2 \log(2/\delta) }{ \abs{\Tau^j_t} } }  \right) \\
    & \quad +  \frac{\RR_{i_*}}{\abs{\Tau_t^{i_*}}^{1/2}} + H \sqrt{\frac{2 \log(2/\delta) }{ \abs{\Tau^{i_*}_t } }}  + H \sqrt{\frac{2 \log(2/\delta) }{ \abs{\Tau^{j}_t } }} \\
}
where we have applied the definition of $\WW$ and event $E_3$ to bound the noise of the returns. Let $C_{\WW}' = \max\{1, C_{\WW}\}$. Since $i_*$ has not been selected yet $\abs{\Tau^{i_*}_t} \geq \frac{t^{1  - \kappa}}{8L } \geq \frac{\abs{\Tau^{j}_t}^{1 - \kappa}}{8L}$. Then, since $\RR_j \leq \RR_{i_*}$,
\eq{
\Delta_{j, i_*} & \leq C_{\WW}' \cdot  \left(  \frac{2\sqrt{8 L }\RR_{i_*}}{\abs{ \Tau^j_t}^{\frac{1 - \kappa}{2}} } 
 + H \frac{ 2\sqrt{16 L \log(2/\delta)} }{\abs{\Tau^j_t}^{\frac{1 - \kappa}{2}}} 
\right)
}
Rearranging gives
\eq{
\abs{\Tau^j_t } & = O \left( \frac{L^{\frac{1}{1  - \kappa}}\left( \RR_{i_*} + H \log^{1/2}(1/\delta) \right)^{\frac{2}{1  - \kappa}} }{\Delta_{j, i_*}^{^{\frac{2}{1  - \kappa}}}}  \right) 
}
Now this bound can be used to bounding the regret with dependence on the gap. The regret during this phase is again
\eq{
\regret_{\tau_{\min}(\delta) :\tau_*} & \leq   H (L - i_*) \tau_*^{1- \kappa} +   \sum_{j < i_*} \sum_{t \in \Tau^j_{\tau_{j + 1}}} V^*_{i_*} - V^{\pi_t} \\
& \leq H (L - i_*) \tau_*^{1- \kappa} +  H i_* +  \sum_{j < i_*} \sum_{t \in \Tau^j_{\tau_{j + 1} - 1}} V^*_{i_*} - V^{\pi_t} 
}
As in the proof of Theorem~\ref{thm::main}, we focus on bounding the right-hand term. For a fixed $j < i_*$, at time $\tau_{j + 1} - 1$ we have that the test succeeds so $\GG_{\tau_{j + 1} - 1}(j, i_*) \leq \WW(\abs{\Tau^j_{\tau_{j + 1} - 1}}, \RR_{i_*}, d_{i_*}, \delta)$. Then, applying the bound on the number of times $j$ can be played, 
\eq{
\sum_{t \in \Tau^j_{\tau_{j + 1} - 1}} V^*_{i_*} - V^{\pi_t}  & \leq \Delta_{j, i_*} \abs{\Tau^j_{\tau_{j + 1} - 1}} + \RR_{j} \cdot \sqrt{\abs{\Tau^j_{\tau_{j + 1} - 1}}}  \\
& \leq O \left( \frac{L^{\frac{1}{1  - \kappa}} \left( \RR_{i_*} + H \log^{1/2}(1/\delta)\right)^{\frac{2}{1  - \kappa}} }{\Delta_{j, i_*}^{{\frac{1 + \kappa}{1  - \kappa}}}}  +  
\frac{\RR_{i_*} L^{\frac{1}{2(1  - \kappa)}} \left(\RR_{i_*} + H \log^{1/2}(1/\delta)\right)^{\frac{1}{1  - \kappa}} }{\Delta_{j, i_*}^{^{\frac{1}{1  - \kappa}}}} \right)  \\
&  = O \left( \frac{L^{\frac{1}{1 - \kappa}} \left(\RR_{i_*} + H \log^{1/2}(1/\delta)\right)^{\frac{2}{1 -\kappa}} }{\Delta_{j, i_*}^{\frac{1 + \kappa}{1 - \kappa}}}    \right) 
}
Therefore, the regret in this phase can be upper bounded by 
\eq{
\regret_{\tau_{\min}(\delta) :\tau_*} & \leq  O \left(  H(L - i_*)T^{1 - \kappa} + Hi_* + L^{\frac{1}{1 - \kappa}} \left( \RR_{i_*} + H\log^{1/2}(1/\delta) \right)^{\frac{2}{1 - \kappa}} \sum_{j < i_*} \frac{1}{\Delta_{j, i_*}^{\frac{1 + \kappa}{1 - \kappa}}} \right) 
}
Combining these three phases, the total regret is
\eq{
O \left(    
H L^{\frac{2}{1 - \kappa}} \log^{\frac{1}{1 - \kappa}}(1/\delta)
+ HL T^{1 - \kappa} + H i_* 
+ L^{\frac{1}{1 - \kappa}} \left( \RR_{i_*} + H\log^{1/2}(1/\delta) \right)^{\frac{2}{1 - \kappa}} \sum_{j < i_*} \frac{1}{\Delta_{j, i_*}^{\frac{1 + \kappa}{1 - \kappa}}}
+ \RR_{i_*} \sqrt{T}
\right) 
}
Choosing either $\kappa = 1/3$ or $\kappa = 1/2$ gives us the statements of Theorem~\ref{thm::id}.
This completes the proof.
\end{proof}

\section{FREEDMAN INEQUALITY}\label{sec::freedman}

In this section, we use a Freedman inequality to lower and upper bound with high probability the number of times a particular algorithm is played both during exploration and while it is chosen by the meta-algorithm (Lemma~\ref{lem::event}).
First, we state a variant of the Freedman inequality from \cite{bartlett2008high}.

\begin{lemma}[Lemma 2,~\cite{bartlett2008high}]\label{lemma::high_probability_freedman}
Suppose $X_1, \cdots, X_T$ is a martingale difference sequence with $| X_s | \leq b$. We define
\begin{equation*}
    \mathrm{Var}_s X_s = \mathbf{Var}( X_s | X_1, \cdots, X_{s-1})
\end{equation*}
Further, let $V_T = \sum_{s=1}^T  \mathrm{Var}_s X_s$ be the sum of conditional variances of $X_s'$s, and $\sigma_T = \sqrt{V_T}$. Then we have, for any choice of $\delta < 1/e$ and $T \geq 4$:
\begin{equation}\label{equation::freedman_inequality_equation}
    \mathbb{P}\left( \sum_{s=1}^T X_s > 2\max(2\sigma_T, b\sqrt{\ln(1/\delta)} )\sqrt{\ln(1/\delta)}    \right) \leq \log_2(T) \delta
\end{equation}
\end{lemma}

Recall that $B_s$ denotes the indices of algorithms that have not been selected by time $s$. Note that $\abs{B_s} \leq L$. For all $i \in [L]$ and $t \in [T]$, define the event
\eq{
\EE_{i, t} := \begin{cases}
 \abs{|\Tau^i_t| - \sum_{s \in [t]} \frac{1}{|B_s| s^\kappa} }  \leq 4 \sqrt{ \sum_{s \in [t]} \frac{1}{s^\kappa}  \log(1/  \delta) }  & \tau_i \geq t \\
 \abs{|\Tau^i_t| - \sum_{s \in [\tau_i]} \frac{1}{|B_s| s^\kappa} - \sum_{s \in [\tau_{i} + 1, t] } \left(1 - \frac{1}{s^\kappa} \right) }   \leq 4 \sqrt{ \sum_{s \in [t]} \frac{1}{s^\kappa}  \log(1/  \delta) }  & \tau_i < t
\end{cases}
}

\begin{lemma}
The event $\EE = \cap_{i \in [L], t \in [T]} \EE_{i, t}$ holds with probability at least $1 - 4 LT^2 \delta \log_2 T$
\end{lemma}

\begin{proof}
Define
\eq{
S_i(t, t') & = \sum_{s \in [t']} Y_{s, i} + \sum_{s \in [t' + 1, t]} \overline Y_{s, i}
}
where $Y_{s, i} \sim \ber\left( \frac{1}{s^\kappa |B_s| }\right)$ and $\overline Y_{s, i} \sim \ber\left(1 - \frac{1}{s^\kappa} \right)$. 
Then define \eq{
Z_i(t, t') & := \sum_{s \in [t]}  \1_{s \leq t'} \cdot \left( Y_{s, i} - \frac{1}{|B_s| s^\kappa} \right) + \1_{s > t'} \left( \overline Y_{s, i} - \left(1 - \frac{1}{s^\kappa} \right) \right)  \\
V_i(t, t') & := \sum_{s \in [t]} \mathbf{Var}_s \left( \1_{t \leq t'} \cdot  \left( Y_{s, i} - \frac{1}{|B_s| s^\kappa} \right) + \1_{t > t'} \cdot \left( \overline Y_{s, i} - \left(1 - \frac{1}{s^\kappa} \right) \right)   \right)
}
where $\mathbf{Var}_s$ denotes the conditional variance up to time $s$.
By definition, $\{ Z_{i}(t, t')\}_{t \geq 1}$ is a martingale sequence and  $V_i(t, t') \leq \sum_{ s\in [t]} \frac{1}{s^\kappa}$. By the Freedman inequality from Lemma~\ref{lemma::high_probability_freedman},
\eq{
\Pr\left(
\abs{ Z_i(t, t') } \geq 4  \sqrt{   \sum_{ s \in [t]} \frac{1}{s^\kappa}  \cdot \log (1 / \delta) }  + 4\log(1/\delta)  
\right) & \leq 2 \delta \log_2 T
}
Let this event be denoted by $\overline \EE_i(t, t')$ for each $i \in [L]$ and $t, t' \in [T]$. Then, by the union bound, the event $\bigcup_{i, t, t'} \overline \EE_i(t,t')$ holds with probability at most $4 LT^2 \delta \log_2 T$.
Therefore, $\bigcap_{t, t' \geq 1} \EE_i(t, t')$ holds with probability at least $1 - 4LT^2 \delta \log_2 T$, and this event implies for all $i \in [L]$ and $t \in [T]$, if $t > \tau_i$, then
\eq{
\abs{|\Tau^i_t| - \sum_{s \in [\tau_i]} \frac{1}{|B_s| s^\kappa} - \sum_{s \in [\tau_{i} + 1, t] } \left(1 - \frac{1}{s^\kappa} \right) }   \leq 4 \sqrt{ \sum_{s \in [t]} \frac{1}{s^\kappa}  \log(1/  \delta) } + 4\log(1/\delta)
}
and if $\tau \leq \tau_i$, then
\eq{
\abs{|\Tau^i_t| - \sum_{s \in [t]} \frac{1}{|B_s| s^\kappa} }  \leq 4 \sqrt{ \sum_{s \in [t]} \frac{1}{s^\kappa}  \log(1/  \delta) }+ 4\log(1/\delta)
}
\end{proof}

\begin{corollary}\label{cor::set-sizes}
With probability at least $1 - 4 LT^2 \delta \log_2 T$, for all $i \in[L]$ and $t \in [T]$ such that $t \geq \tau_{\min}(\delta)$, the following is true:
\elist{
    \item If $t \leq \tau_i$, then $\frac{t^{1 - \kappa}}{8L} \leq \abs{\Tau^i_t} \leq 4t^{1-\kappa}$.
    \item If $t > \tau_i$, then $\abs{\Tau^i_t} \leq t - \tau_i + 4 t^{1 - \kappa}$.
}
\end{corollary}
\begin{proof}
Note that when $t \leq \tau_i$, it is also the case that $\abs{B_s} \geq 1$ for all $s \leq t$. We condition on the event $\EE$ from above, which occurs with probability at least $1 - 4 LT^2 \delta \log_2 T$.
Given this event, it follows that if $t \leq \tau_i$, then
\eq{
\abs{\Tau^i_t} & \geq \sum_{s \in [t]}\frac{1}{s^\kappa |B_s| } - 4 \sqrt{ \sum_{s \in [t]} \frac{1}{s^\kappa}  \log(1/  \delta) } - 4 \log(1/\delta) \\
 & \geq  \frac{1}{2L}\sum_{s \in [t]}\frac{1}{s^\kappa  } - 32 L \log(1 /\delta) \\
 & \geq \frac{1}{2L} \left( t^{1 - \kappa} - 2 \right)  - 32 L \log(1 /\delta) \\
 & \geq \frac{t^{1 - \kappa}}{4L}    - 32 L \log(1 /\delta) \\
 & \geq \frac{t^{1 - \kappa}}{8L} 
} 
The second inequality uses the AM-GM inequality along with the fact that $|B_s| \leq L$, which implies
\eq{
\sqrt{ \sum_{s \in [t]} \frac{1}{L s^\kappa}  \cdot  16L  \log(1/  \delta) } & \leq \frac{1}{2L}\sum_{s \in [t]}\frac{1}{s^\kappa  } + 8L  \log(1/  \delta)
}
The third applies the integral approximation of the sum.  The last two follow from the condition that $t \geq \tau_{\min}(\delta) = C_{\min} \cdot L^{\frac{2}{1 - \kappa}} \log^{\frac{1}{1 - \kappa}} (1/\delta)$ for a large enough constant $C_{\min} > 0$.

The other side follows similarly with 
\eq{
\abs{\Tau_t^i} & \leq 3 t^{1-\kappa} + 32 \log(1/\delta)  \\
& \leq 4 t^{1 - \kappa}
}
when $t \geq \left(32 \log(1/\delta)\right)^{\frac{1}{1 - \kappa}}$.

Similarly, for $t > \tau_i$, event $\EE$ guarantees
\eq{
\abs{\Tau^i_t} & \leq \sum_{ s\in [\tau_i]} \frac{1}{s^\kappa |B_s|}  + \sum_{s \in [\tau_i + 1, t]} \left(1 - \frac{1}{s^\kappa}\right) + 4 \sqrt{ \sum_{s \in [t]} \frac{1}{s^\kappa}  \log(1/  \delta) } + 4 \log(1/\delta) \\
& \leq  t - \tau_i + 32 \log(1/\delta) + \frac{3}{2} \sum_{ s\in [\tau_i]} \frac{1}{s^\kappa } \\
& \leq t - \tau_i + 32 \log(1/\delta) + 3 t^{1-  \kappa} \\
& \leq t - \tau_i + 4 t^{1-  \kappa}
}
when $t \geq \tau_{\min}(\delta)$.

\end{proof}

\section{APPLICATIONS}\label{sec::applications}

In this section, we expand on the applications of Theorem~\ref{thm::main} to paradigms of function approximation in RL.

\paragraph{Linear MDPs} Consider the setting of \cite{jin2019provably} which we mentioned as an example in Section~\ref{sec::setting}.  
In this setting, we assume access to a set of nested features $\phi_{i} : \SSS \times \AA \to \R^{d_i}$ for $i \in [L]$ such that $d_{i} \leq d_{i + 1}$ and the first $d_i$ components of $\phi_{i + 1}$ are the same as $\phi_i$. These features generate linear model classes of the form
\eq{
\FF_i = \left\{ (s, a) \mapsto \< \phi_i(s, a), \theta\> \ : \ \theta \in \R^{d_i} \right\}
}
Nested-ness of the features ensures that $\FF_i \subseteq \FF_{i + 1}$ for all $i$. In accordance with the setting of \cite{jin2019provably}, we assume that there exists some minimal $i_*$ such that for any $\FF_i$ with $i \geq i_*$ there exist $\mu(\cdot)$ and $\omega_{i, h} \in \R^{d_i}$ that predict exactly the transition probabilities $P$ and reward $r$:
\begin{align}\label{eq::lsvi-ucb-assumption}
\begin{split}
	P(s' | s, u) & = \< \phi_i(s, u), \mu_i(s') \> \\
	r_h(s, u) & = \< \phi_i(s, u), \omega_{i, h}\>
\end{split}
\end{align}
Here, $\mu_i(\cdot)$ is a $d_i$-dimensional vector of measures on $\SSS$.
Let $\{ \AA_i\}$ be instances of LSVI-UCB equipped with the doubling trick and model classes $\{ \FF_{i}\}$. We further assume that the features and parameters for each of the models with $i \geq i_*$ satisfies the regularity conditions of Assumption A of \cite{jin2019provably}, i.e. bounded $\ell_2$ norms, $r \in [0, 1]$.

\cite{jin2019provably} guarantees that for $i \geq i_*$ and $t \in [T]$ with probability at least $1 - \delta_0$, $\regret_t(\AA_i) = O( \sqrt{d_i^3 H^4 t \cdot \log^2(d_i TH/\delta_0)})$. Adapting this to the framework of \ms{}, we let $\RR_{i} = O \left( \sqrt{d_i^3 H^4 \cdot \log^2(d_i TH/\delta)} \right)$, which ensures $\RR_{i} \leq \RR_{i + 1}$.
A model selection corollary immediately follows from Theorem~\ref{thm::main}.
\begin{corollary}
In the linear MDP setting of (\ref{eq::lsvi-ucb-assumption}) with LSVI-UCB, \ms{} guarantees with probability at least $1 - \delta'$
\eq{
\regret_T & = \widetilde O\left(  \sqrt{d_{i_*}^3 H^4 \log^2(d_{i_*} LTH/\delta') } \cdot L^{5/6} T^{2/3} \right)
}
\end{corollary}

\cite{yang2019reinforcement} consider a similar setting of linear MDPs where the transition dynamics $P$ are linear. We again assume access to nested linear models but of the form
\eq{\FF_{i} = \left\{(s, u, s') \mapsto \phi_i(s,u)^\top M \psi_i(s')  \ : \ M \in \R^{d_i \times d_i'}\right\}
}
where $\{ \phi_i\}_{i \in [L]}$ and $\{ \psi_i \}_{i \in [L]}$ are nested features of dimension $d_i$ and $d_i'$ respectively. \cite{yang2019reinforcement} assume that there is some minimal $i_*$ such that for any $i \geq i_*$, there is $M \in \R^{d_i \times d_i'}$ such that
\begin{align}\label{eq::matrixrl-assumption}
P(s ' | s, u) & = \phi_i(s, u)^\top M \psi_i(s') 
\end{align}
for all $s, s' \in \SSS$, $u \in \UU$. We further adhere to the regularity conditions of Assumption 2 of \cite{yang2019reinforcement}, who guarantee the MatrixRL $\AA_i$ with model $\FF_{i}$ has regret $\regret_t(\AA_{i}) = \widetilde O \left( \sqrt{d^3_i H^5 t} \cdot \log(d_i TH / \delta_0) \right)$ with probability at least $1 - \delta_0$.
Letting $\RR_i = \widetilde O \left( \sqrt{d_i^3 H^5 } \cdot \log(d_i T H /\delta) \right)$, we have the following model selection guarantee.
\begin{corollary}
In the linear MDP setting of (\ref{eq::matrixrl-assumption}) with MatrixRL, \ms{} guarantees with probability at least $1 - \delta'$
\eq{
\regret_T & = \widetilde O\left(  \sqrt{d_{i_*}^3 H^5 \log^2(d_{i_*} LTH/\delta') } \cdot L^{5/6} T^{2/3} \right)
}
\end{corollary}

The final linear setting we consider is that of low inherent Bellman error studied by \cite{zanette2020learning}.
We let $\FF_i$ be defined as it is in (\ref{eq::linear-model}) and let $\BB = \{ \theta \in\R^{d_i} \ : \ \| \theta\| \leq D\}$ for some $D > 0$.  Then assume there is a minimal $i_*$ such that for any $i \geq i_*$ and $\theta_{h + 1} \in \BB$, there is $\theta_{h}$ such that
\eq{
\<\phi_i(s, u), \theta_h\> - \B_{h} Q_{h + 1}(\theta_{h + 1})(s, u) = 0
}
for all $s \in \SSS$ and $u \in \UU$, where $Q_{h}(\theta)$ is the linear action-value function parameterized by $\theta$ (with features $\phi_i$) and $\B_h$ is the Bellman operator with reward $r_h$. In other words, this condition asserts that $\FF_{i_*}$ has zero inherent Bellman error. Under the same regularity conditions, for $i \geq i_*$, \cite{zanette2020learning} guarantees ELEANOR achieves $\regret_t(\AA_i) = \widetilde O \left( d_i \sqrt{H^4 t} \right)$ with probability at least $1 - \delta_0$. 
Letting $\RR_i = \widetilde O \left( d_i\sqrt{ H^4 } \right)$, we have the following model selection guarantee.
\begin{corollary}
In the inherent Bellman error setting with ELEANOR, \ms{} guarantees with probability at least $1 - \delta'$
\eq{
\regret_T & = \widetilde O\left(  d_{i_*}\sqrt{ H^4  } \cdot L^{5/6} T^{2/3} \right)
}
\end{corollary} 
where $\widetilde O$ hides polylog dependencies.

\paragraph{Low Bellman Rank} Another class of algorithms using more general function approximation considers the setting of MDPs with low Bellman rank \citep{jiang2017contextual}. In this setting, a finite model class $\FF : \SSS \times \UU \to \R$ realizes $\MM$ if there exists $f^* \in \FF$ such that $Q^{*}_h(s, a) = f^*(s, a)$, where $Q^*$ is the optimal action-value function for all $h\in[H]$. For any $f \in \FF$, define $\pi_f$ as the greedy policy with respect to $f$, and the Bellman error at $h \in [H]$ as
\eq{
\EE(f, \pi, h) := \E \left[ f(s, \pi_f(s)) - r(s, \pi_f(s)) - f(s', \pi_f(s'))   \right] ,
}
where the expectation is over $s$ from the state distribution of $\pi$ at $h$ and $s' \sim P(\cdot | s, \pi_f(s))$.
In this setting, it is assumed that there is a Bellman rank $M \ll |\FF|$ such that for any $f, g \in \FF$, we have $\EE(f, \pi_g, h) = \< \nu_h(g), \xi_h(f)\>$ for $\nu_h(g) , \xi_h(f) \in \R^M$ and $\|\nu\| \| \xi\| \leq \zeta$. We assume access to a set of finite model classes $\{ \FF_{i} \}_{i \in [L]}$ such that there is at least one that realizes $\MM$, and the complexity of $\FF_i$ is a function of its cardinality $| \FF_i |$ and induced Bellman rank $M_i$.
We consider instances of the AVE algorithm $\{ \AA_i \}$ of \cite{dong2020n} with the doubling trick, which has nominal regret $\widetilde{O}  \left( \sqrt{ M_{i}^2|\UU| H^4 t  \log^3 |\FF_{i} | } \right)$.
Choose $\RR_{\FF_i} = \widetilde{O} \left( \sqrt{M_i^2 |\UU| H^4\log^3(| \FF_{i} | )} \right)$ and let $i_*$ be the smallest index that realizes $\MM$.
This yields the following corollary. 

\begin{corollary}
In the low Bellman rank setting with AVE, the model selection algorithm guarantees with probability at least $1 - \delta'$
\eq{
\regret_T(\AA) & = \widetilde O\left(  \sqrt{M_{i_*}^2 |\UU| H^4\log^3(| \FF_{i_*} | )} \cdot L^{5/6} T^{2/3} \right) .
}
\end{corollary}

\section{Implications of fast rates of estimating $V^*$ and/or gap between policy classes}\label{sec::v-star}

We previously discussed the recent results that prove PAC~\citep{modi2020sample} and regret~\citep{pacchiano2020model} results for model selection in RL given knowledge of $V^*$. We now show an analogous result for our setting. 
We use the framework of Algorithm~\ref{alg::ms} but set the probability of forced exploration to zero, i.e. set $\kappa = \infty$.
Then, the test is modified to check the following condition for eliminating model $\hat \imath_t$:
\eq{
\sum_{t' \in \Tau^{\hat \imath_t}_t} V^*  -  g_{t'} > \WW_{V^*}(\abs{\Tau^{\hat \imath_t}_t}, \RR_{\hat \imath}, d_{\hat\imath_t}, \delta)
}
where
\eq{
\WW_{V^*} (\Delta, \RR, d, \delta) & = C_{\WW} \cdot   \RR(d, H, \log(1/\delta)) \cdot \sqrt{\Delta}  \\
& \quad + C_\WW \cdot H \sqrt{  \Delta \cdot \log (1/\delta)  }
}
for a sufficiently large constant $C_{\WW_{V^*}} > 0$. The test effectively measures the regret of $\AA_{\hat \imath_t}$ up to noise in $g_t$ and rejects when we are confident that the regret does not match the nominal.

\begin{proposition}
Given side information of the optimal value $V^*$ for MDP $\MM$, the above model selection algorithm $\AA$ guarantees regret
\eq{
\regret_T(\AA) = \widetilde O \left( \RR_{i_*}(d_{i_*}, H, \log(LT/\delta')) \cdot \sqrt{LT}  \right)
}
with probability at least $1 - \delta'$.
\end{proposition}

\begin{proof}
The proof is identical to that of Theorem~\ref{thm::main} except for the handling of the misspecified case. For any model $j < i_*$ for which there is a time when the test succeeds, 
\eq{
\sum_{t \in \T^j_{\tau_{j + 1} - 1}} V^* - V^{\pi_t} & = \sum_{ t \in \T^j_{\tau_{j + 1} - 1}} (V^* - g_t ) + \sum_{t \in \T^j_{\tau_{j + 1} - 1}} \epsilon_t \\
& \leq \WW_{V^*}(\abs{\Tau^{j}_t}, \RR_{j}, d_{j}, \delta) + \sum_{t \in \T^j_{\tau_{j + 1} - 1}} \epsilon_t \\
& = O \left( \left( \RR_{i_*} + H \log^{1/2}(1/\delta) \right) \cdot \sqrt{  \abs{\Tau^{j}_t} } \right)
}
Summing over all $j < i_*$ and using Jensen's inequality again shows that the dominant term remains $O (\RR_{i_*} \sqrt{T})$ instead of $O (  \RR_{i_*} T^{2/3} )$.
\end{proof}

This regret optimally matches the regret of the base algorithms in both $\RR_{i_*}$ and $T$, but a dependence on $L$ is still included.

Unfortunately, it is unclear whether such an assumption of knowing $V^*$ is realistic in practice. An immediate alternative solution is to try to estimate $V^*$ without first finding the optimal policy. The original test in Section~\ref{sec::test} attempts this: the average returns of the algorithms in $B_t$ act as a noisy lower bound of $V^*$. The test, however, is sensitive to the amount of exploration allocated to the base algorithms, and, since we are comparing to the nominal regret, the flat dependence on $\RR$ is unlikely to improve. We hypothesize that better estimates of $V^*$ can significantly improve the model selection guarantee.

In the following subsections, we consider the implications of having access to fast estimators, either of the optimal value $V^* := V^*_{i_*}$ or \textit{gaps} between optimal values of different model orders, i.e. $\Delta_{i,j} := V^*_i - V^*_j$.
We employ our instance-dependent analysis to show that improved regret rates can be obtained in both cases when the gap between the value of the optimal policy class and others is relatively large (i.e. constant).
These consequences are demonstrated for the special case of linear contextual bandits, where such fast estimators are known to be available~\citep{dicker2014variance,verzelen2018adaptive,kong2018estimating,kong2020sublinear}.

\subsection{Implications for access to a fast rate of estimating \textit{gaps} in policy class optimal values}

\begin{figure}
	\begin{algorithm}[H]
		\caption{ Explore-Commit-Eliminate With Fast Gap Estimator And Forced Exploration Routines(\msgap{}) } \label{alg::ms2}
		\begin{algorithmic}[1]
			\STATE \textbf{Input}: $\{ \AA_i, \widetilde{\AA}_i, \FF_i,\VV_i, d_i \}_{i \in [L]}, T, \delta',\tau_{\min}(\cdot)$
			\STATE $\delta \leftarrow \frac{\delta'}{10LT^2 \log_2 T}$,  $\hat \imath_t \leftarrow 1$,  $\Tau^i_1 = \emptyset$ for $i \in [L]$, $B_1 = [2, L]$
			\STATE $U_t = \begin{cases}
										0 & \text{w.p. } 1 - \frac{1}{t^\kappa} \\
					1 & \text{w.p. } \frac{1}{t^\kappa} 
			\end{cases}$ for all $t \in [T]$.
			\FOR{ $t = 1, \ldots, T$ }
		
			\IF{ $U_t = 0$ }
				\STATE Set $j \leftarrow \hat\imath$.
			\ELSE
				\STATE Sample $J_t \sim \unif\{ B_t \}$
				\STATE Set $j \leftarrow J_t$
			\ENDIF
				\STATE $\Tau^{j}_t \leftarrow  \Tau^{j}_t \cup \{ t\}$ and $\Tau^{k}_t \leftarrow \Tau^{k}_t$ for all $k \neq j$.		
				\STATE IF $U_t = 0$: Rollout policy $\pi_t$ from $\AA_{j}$.
				\STATE ELSE: Rollout policy $\pi_t$ from $\widetilde{\AA}_j$.
				\STATE Observe $z_t := (s_{t, 1}, u_{t, 1}, \ldots, u_{t, H}, s_{t, H + 1})$ and $g_t := \sum_{h \in [H]} r_{t, h}$
				\STATE Update $\AA_{j}$ if $U_t = 0$, else update $\widetilde{\AA}_j$ with $t, z_t, g_t$
				\IF{ $t \geq \tau_{\min}(\delta)$ and there exists $j \in B_t$ such that $\widehat{\Delta}_{\hat\imath_t, j}(\Tau_t^j) > \ZZ(  | \Tau_{t}^{j}|, \VV_{j})$}
					\STATE $\hat \imath_{t + 1} \leftarrow \hat \imath_t + 1$
					\STATE $B_{t + 1} \leftarrow B_{t} \setminus \{ \hat \imath_{t+1} \}$
					\STATE If $\hat \imath_{t + 1} = L$, break and run $\AA_L$ to end of time
				\ELSE 
					\STATE $B_{t+1} = B_t$
				\ENDIF
\ENDFOR
		\end{algorithmic}
	\end{algorithm}
\vspace{-.75cm}
\end{figure}

We first consider the possibility of fast rates in estimating the \textit{gap} in optimal policy values, i.e. $\Delta_{i, j} := V^*_j - V^*_i$ for all $i < j$.
In this section, we show that a modification of our  $\mathsf{ECE}$ algorithm with a direct estimator of the gap in maximal values would yield improved model selection rates if there is a constant gap between all lower-order models and the true model, i.e. $\Delta_{i,i_*} > 0$ for all $i$.
Along with the replaced estimator, the radius of the statistical test is also modified according to the faster estimation error rate in the policy gap.
For the special case of linear contextual bandits, these modifications will correspond \textit{exactly} to the \textsf{ModCB} algorithm proposed by~\citet{foster2019model}.

Since our focus is on instance-dependent analysis, we carry over the assumptions from Section~\ref{sec::id}, and further assume model nested-ness in the sense that $V_j^* = V^*$ for $j \geq i_*$.
Thus, we get $\Delta_{i_*,i} = 0$ for all $i \geq i_*$, and $\Delta_{i, i_*} > 0$ for all $i < i_*$. To estimate the gap during exploration episodes, rather than running $\AA_i$ directly, we allow an exploration algorithm $\widetilde \AA_i$ to be run. In the case of \cite{foster2019model} for contextual bandits, this would be an exploration policy that picks an arm uniformly at random from the set of $K$ arms.
Finally, we make the following assumption on the estimation error rate of the gaps.

\begin{assumption}\label{as:fastrate}
For any $i < j$, we define $\widehat{\Delta}^{(n)}_{i,j}$ as an estimate of $\Delta_{i,j}$ that is a functional of the (context and reward) feedback obtained after running $n$ exploration episodes for $\widetilde{\AA}_j$.
Then, we say that our estimate is $\mathcal{V}_j := \mathcal{V}(d_j,H,\log(1/\delta))$-consistent if, for some positive constant $C > 1$, we have 
\begin{align}\label{eq:fastrate}
    |\widehat{\Delta}_{i,j}^{(n)} - \Delta_{i,j}| \leq \frac{\Delta_{i,j}}{C} + \frac{\VV_j}{\sqrt{n}} \text{ for all } n \in [T] \text{ and } i < j 
\end{align}
with probability at least $1 - \delta$.
As with the earlier definition\footnote{Similar to $\RR$, the definition of $\VV_j$ can be general and include other problem dependent parameters as well.}, $\mathcal{V_j}$ is poly and non-decreasing in $d_j$, $H$, $\abs{\UU}$, and $\log(LT/\delta))$.
\end{assumption}

The original estimator used in the $\ms$ algorithm satisfies the above assumption with $\mathcal{V} := \RR$.
In what follows, we want to exploit situations in which we have available an estimator $\widehat{\Delta}_{i,j}$ with guarantee $\mathcal{V} \ll \mathcal{R}$; in particular, the dependence of the function $\mathcal{V}$ on dimension $d$ could be significantly improved over any regret bound.
While constructing such estimators is in general a open problem in RL, we do have one example for the linear contextual bandit problem where this is known to be possible.

\begin{restatable}{example}{linearCBexample}[Linear contextual bandits.]
Consider the stochastic $d^{th}$-order linear contextual bandits model as in~\cite{chu2011contextual}, parameterized by $K$ context distributions $\{\Sigma_i\}_{i=1}^K$, reward parameter $\theta^* \in \mathbb{R}^d$, and $\sigma$-sub-Gaussian noise in the rewards.
Further, we carry over the assumptions from~\cite{foster2019model} of $\tau$-sub-Gaussianity of the contexts and $\lambda_{min}(\overline{\Sigma}) \geq \nu > 0$ where $\overline{\Sigma} := \frac{1}{K} \sum_{i=1}^K \Sigma_i$ is the action-averaged covariance matrix.
We assume that $\tau,\nu$ are universal positive constants.
Then, Assumption~\ref{as:fastrate} holds with the choice of forced exploration $\widetilde{\AA}_i$ that chooses arms uniformly at random from the set $[K]$ (regardless of round index $t$ and model index $i$), with the choices $C = 2$ and $\mathcal{V}_i(d_i, \log(1/\delta))$ scaling as $\widetilde{O}(d_i^{1/4})$ for the estimator based on the square loss gap, used in~\cite{foster2019model}. Meanwhile, the regret bound for the base algorithms (e.g. instances of Exp4-IX) would give $\RR_i$ scaling as $\widetilde O (d_i^{1/2})$.
Further, note that Algorithm 2 exactly becomes the \textsf{ModCB} algorithm for this case.
\end{restatable}

We now described the modified $\mathsf{ECE}$ algorithm, \msgap{}, to work with a plugged-in estimate of $\Delta_{i,j}$ with the above guarantees.
Note that the input now has extra ``exploration algorithms" $\widetilde{\AA}_i$, and what was earlier defined as regret bound leading factors, i.e. $\RR_i$, is replaced by $\VV_i$, the leading factors in the gap estimation error.
Importantly, we are now using the fast estimator $\widehat{\Delta}_{i,j}(t)$ in place of the earlier estimator $\GG_t(j,i)/|\Tau_t^j|$.

Moreover, the threshold is now defined as:
\eq{
\ZZ(n, \VV)  & := \frac{\VV}{\sqrt{n}}
}

Note that the threshold
\textit{is always applied to the more complex model} $d := d_i$ for $i > j$.
The algorithm is stated formally in Algorithm~\ref{alg::ms2}.
We derive the following instance-dependent result for this algorithm. 

\begin{proposition}\label{prop:fastrate}
For a given $\MM$, let Assumption~\ref{as:fastrate} hold and let $\{ \Delta_{i, i_*}\}_{i < i_*}$ be the gaps.
Then, with probability at least $1 - \delta'$, \msgap{} in Algorithm~\ref{alg::ms2} satisfies the regret bound 
\begin{align*}
    \widetilde{{O}}\left(HLT^{1 - \kappa} + \RR_{i_*}^{\Pi_{i_*}} \sqrt{LT} + \sum_{i = 1}^{i_* - 1} \min\{ L^{\frac{1}{1 - \kappa}} \VV_{i_*}^{\frac{2}{1 - \kappa}} \Delta_{i, i_*}^{- \frac{1 + \kappa}{1 - \kappa}}, \Delta_{i, i_*} T\}\right) ,
\end{align*}

where regret is measured with respect to the optimal value $V^*$.
\end{proposition}

Before proving Proposition~\ref{prop:fastrate}, let us consider its implication for the linear contextual bandits setting, ignoring dependence on $K = \abs{\UU}$ for now.
Here, the modified \textsf{ECE} algorithm will essentially correspond to \textsf{ModCB}.

By choosing $\kappa = 1/3$ and using the gap estimator from \cite{foster2019model}, we can achieve an instance-dependent result with lower $d_{i_*}$ dependence than that of Theorem~\ref{thm::id} for the same setting of $\kappa$ under the assumption of constant gaps. Furthermore, in the case the case of variable gaps, this result can immediately imply a minimax guarantee that matches that of \cite{foster2019model}.

\begin{corollary}
For the linear contextual bandit problem, under the same setting as Corollary~\ref{cor::optimal-id}, with probability at least $1 - \delta'$, Algorithm~\ref{alg::ms2} with $\kappa = 1/3$ and constant gaps satisfies the instance-dependent regret bound
\begin{align}\label{eq::gap-instance1}
\widetilde{{O}} \left(LT^{2/3} + \sqrt{d_{i_*}L T} + L^{3/2} d_{i_*}^{3/4} \sum_{i < i_*} \Delta_{i,i_*}^{-2}\right) =
    \widetilde{{O}}\left(LT^{2/3} + \sqrt{d_{i_*}L T}\right).
\end{align}
Furthermore, for variable gaps, let $\regret_T(\AA; \MM, \{ \Delta_{i, i_*}\}_i)$ denote the regret as a function of the gaps. Since $\min\{L^{3/2}\VV_{i_*}^3 \Delta_{i_*,i}^{-2}, \Delta_{i_*,i} T\} \leq L^{1/2}\VV_{i_*} T^{2/3}$, \msgap{} also satisfies the minimax regret bound
\eq{
\sup_{\Delta_{i, i_*} > 0 \ : \ i < i_*} \regret_T\left( \msgap; \MM, \{\Delta_{i, i_*}\}_i\right) = \widetilde {O}\left( Ld_{i_*}^{1/4} T^{2/3} + \sqrt{d_{i_*}LT} \right).
}
\end{corollary}

The equality in the (\ref{eq::gap-instance1}) uses $d_i \ll T$ for all $i \in [L]$ and the constant gap assumption.  If we knew \textit{a priori} that the gaps are constant, the instance-dependent bound in (\ref{eq::gap-instance1}) can be improved by a more aggressive choice of $\kappa = 1/2$, as in Theorem~\ref{thm::id}. We can then achieve the desired regret rate of $\widetilde O(\sqrt{d_{i_*}T})$ regret \textit{if and only if} the gaps are constant. Again there is only sub-optimal $d_{i_*}$-dependence on the term independent of $T$.

\begin{corollary}\label{cor::optimal-id}
For the linear contextual bandit problem under Assumption~\ref{as:fastrate} with constant gaps $\{\Delta_{j, i_*} \}_{j < i_*}$, let $\VV_{i_*} := \widetilde O (d_{i_*}^{1/4})$ and $\RR_{i_*}^{\Pi_{i_*}} := \widetilde O ( d_{i_*}^{1/2})$.  Then, with probability at least $1 - \delta'$, Algorithm~\ref{alg::ms2} with $\kappa = 1/2$ satisfies the regret bound
\eq{
\widetilde{{O}} \left(L\sqrt{T} + \sqrt{d_{i_*}L T} + L^{2} d_{i_*} \sum_{i < i_*} \Delta_{i,i_*}^{-3}\right) =
    \widetilde{{O}}\left(L\sqrt{T} + \sqrt{d_{i_*}L T}\right).
}
\end{corollary}

In summary, Proposition~\ref{prop:fastrate} not only recovers the minimax rate, but shows an improved instance-dependent guarantee for more favorable cases when the gap between optimal policy values is larger.

Let us now prove the proposition.
\begin{proof}

Let $\widehat\Delta_{i, j}^t := \widehat\Delta_{i, j}^{(\abs{\Tau_t^i})}$. First, we show that under the intersection of the event of Equation~\eqref{eq:fastrate} and event $E'$ of Theorem~\ref{thm::id}, we will never reach $\hat \imath_t > i_*$.
For every $i > i_*$, and all $t \geq 1$, Equation~\eqref{eq:fastrate} gives us
\begin{align*}
    \widehat{\Delta}_{i_*,i}^t \leq \frac{\VV_i}{\sqrt{|\Tau_t^i|}}
\end{align*}
Thus, model order $i_*$ is never rejected under this event, and higher order models have no contribution to the overall regret.

Next, we bound the regret arriving from the misspecified models $i < i_*$.
We do this by bounding the number of rounds during which model order $i < i_*$ is used, given by $|\Tau_T^i|$.
From Equation~\eqref{eq:fastrate}, we get
\begin{align*}
    \Delta_{i, i_*} &\leq \widehat{\Delta}_{i, i_*}^{t} + \frac{\Delta_{i, i_*}}{C} + \frac{\VV_{i_*}}{\sqrt{|\Tau_t^{i_*}|}}\\
    \implies \Delta_{i, i_*} &\leq \frac{C}{C - 1} \left( \widehat{\Delta}_{i_*,i}^t + \frac{\VV_{i_*}}{\sqrt{|\Tau_t^{i_*}|}}\right) \\
    &\leq \frac{C\VV_{i_*}}{(C-1)\sqrt{|\Tau_t^{i_*}|}}
\end{align*}

where the last inequality follows because the condition in the test has not yet been violated.
More-over, since model $i_*$ has not been selected yet, we have $|\Tau_t^{i_*}| \geq \frac{t^{1 - \kappa}}{8L} \geq \frac{|\Tau_t^i|^{1 - \kappa}}{8L}$.
This gives us
\begin{align*}
    \Delta_{i, i_*} &\leq \frac{8(CL)^{1/2} \VV_{i_*}}{\sqrt{C - 1}|\Tau_t^i|^{\frac{1 - \kappa}{2}}} \\
    \implies |\Tau_t^i| &= \mathcal{O}\left(\frac{L^{\frac{1}{1 - \kappa}} (\VV_{i_*})^{\frac{2}{1 - \kappa}}}{\Delta_{i, i_*}^{\frac{2}{1 - \kappa}}}\right)
\end{align*}

Thus, the total contribution to the regret from the misspecified model $i$ is given by
\begin{align*}
    &T^{1 - \kappa} + |\Tau_t^i| \Delta_{i, i_*} + \RR_i^{\Pi_{i}} \sqrt{|\Tau_t^i|} \\
    &\leq T^{1 - \kappa} + |\Tau_t^i| \Delta_{i, i_*} + \RR_{i_*}^{\Pi_{i_*}} \sqrt{|\Tau_t^i|} .
\end{align*}

The first term comes from the forced exploration, and the last term is equivalent to the regret we would pay anyway if we knew $i_* = 2$ beforehand.
Focusing on the second term, the contribution to regret is upper bounded by
\begin{align*}
    \min\left\{\Delta_{i,i_*} T, \left(\frac{C_\ZZ L^{1/2} \VV_{i_*}}{\Delta_{i,i_*}}\right)^{\frac{2}{1 - \kappa}} \cdot \Delta_{i,i_*}\right\}
\end{align*}
\end{proof}

\subsection{Implications for a fast rate of estimating $V^*$}

An alternative setting is one where we have access to an estimator of $V^*$ instead of an estimator of the gap. Corollary 1 of \cite{kong2020sublinear} shows that an $\epsilon$-close approximation of $V^*$ is possible in $\widetilde O \left(\sqrt{d}/\epsilon^2\right)$ interactions in the disjoint linear bandit setting (where there is a different parameter vector for each arm) under Gaussian assumptions.  
Whether or not such fast estimators exist or are practical for other general settings is still open, but future work on this problem could be applied to the instance dependent results here.

We will retain the same problem assumptions as the previous subsection.
We also assume there is $\widehat V_i$ for each $i \in [L]$. Each estimator offers a high-probability guarantee on the estimation error as a function of the number of exploration episodes using corresponding exploration algorithms $\{\widetilde \AA_i\}$.
\begin{assumption}\label{asmpt::vstar-estimator}
For all $i \in [L]$, we define the  $\widehat V_i^{(n)}$ where $n \in [T]$ as the estimator of $V_i^*$ given $n$ exploration rounds with $\widetilde \AA_i$. We assume with probability at least $1 - \delta$, for all $i \geq i_*$, the estimator $\widehat V_i^{(n)}$ satisfies
\begin{align}\label{eq::vstar-estimator-bound}
\abs{V^* - \widehat V_i^{(n)}} & \leq \frac{\VV_i}{n^{\alpha}} + \frac{\VV_i'}{n^{\beta}}
\end{align}
where $\VV_i$ and $\VV_i'$ are poly and increasing in $d$, $H$, $\abs{\UU}$, and $\log(LT/\delta))$ and $\alpha, \beta \in (0, 1)$.
\end{assumption}

Let $\hat V_i^t := \widehat V_i^{(\abs{\Tau_{t}^i})}$. The algorithm will be of the same form as Algorithm~\ref{alg::ms2}, but instead we leverage the following alternative test:
\begin{equation}
\label{eqn:alt_test_V}
\sum_{t \in \Tau^{\hat\imath_t}_t} \widehat V_{j}^t - g_{t'} \leq \ZZ_{\hat \imath} (\abs{\Tau^{\hat\imath_t}_t}, \VV_j, \VV'_j)
\end{equation}
where \eq{
\ZZ_i (t, \VV, \VV') := C_\ZZ \left( \VV_j L^\alpha t ^{1 - (1 - \kappa)\alpha} + \VV_j' L^\beta t ^{1 - (1 - \kappa)\beta} + H \sqrt{t \log(1/\delta) } + \RR_i^{\Pi_i} \sqrt{ t} \right)
}
for a sufficiently large constant $C_\ZZ > 0$. That is, if the above inequality holds, then \ms{} continues to use $\hat\imath_t$; otherwise, \ms{} switches to $\hat\imath_t + 1$ for round $t + 1$. First, we prove an analogous result to Lemma~\ref{thm::test}, showing that the test will not fail under the good event $E''$. Here, we let $E'' = E' \cap E_4$ where $E'$ is the event from Theorem~\ref{thm::id} and event $E_4$ is the following. 

Event $E_4$: Let $\{\widehat V_i\}$ be the estimators from Assumption~\ref{asmpt::vstar-estimator}. For all $i \geq i_*$ and $n \in [T]$, equation (\ref{eq::vstar-estimator-bound}) is satisfied. 

Note that $E_4$ holds with probability at least $1 - \delta$ by assumption. Therefore $E''$ still holds with probability at least $1 - 10LT^2\delta\log_2(T)$.

\begin{lemma}
Given that event $E'$ holds, then for all $t \geq \tau_{\min}$ and $j \in [i_* + 1, L]$, it holds that $\sum_{t' \in \Tau^{i_*}_t} \hat V^{j}_t - g_{t'} \leq \ZZ_{i_*}(\abs{\Tau^{i_*}_t}, \VV_j, \VV'_j)$
\end{lemma}
\begin{proof}
Since $j > i_*$, we use the assumption on the estimator $\widehat V_j$ to write the difference in terms of regret, estimation error and noise:
\eq{
\sum_{t' \in \Tau^{i_*}_t} \widehat V_{j}^t - g_{t'} & \leq \sum_{t' \in \Tau^{i_*}_t} \widehat V_{j}^t - V^{\pi_{t'}} - \epsilon_{t'} \\
& \leq \frac{\VV_j \abs{\Tau^{i_*}_t} }{\abs{\Tau^{j}_t}^\alpha}  + \frac{\VV_j' \abs{\Tau^{i_*}_t} }{\abs{\Tau^{j}_t}^\beta} +  \sum_{t' \in \Tau^{i_*}_t}  V^* - V^{\pi_{t'}} - \epsilon_{t'}
}
Then note that $\sum_{t' \in \abs{\Tau^{i_*}_t}} \epsilon_{t'} \leq  H \sqrt{2 \abs{\Tau^{i_*}_t } \log(2/\delta) }$ and $\sum_{t' \in \abs{\Tau^{i_*}_t}} V^* - V^{\pi_{t'}} \leq \RR_{i_*}^{\Pi_{i_*}} \sqrt{ \abs{\Tau^{i_*}_t} }$ under event $E'$. Furthermore, under $E'$, we have $\abs{\Tau^{j}_t} \geq \frac{t^{1 - \kappa}}{8L} \geq \frac{\abs{\Tau^{i_*}_t}^{1 - \kappa}}{8L}$, which implies
\eq{
\sum_{t' \in \Tau^{i_*}_t} \widehat V_{j}^t - g_{t'} & \leq C_\ZZ \left( \VV_j L^\alpha\abs{\Tau^{i_*}_t}^{1 - (1 - \kappa)\alpha} + \VV_j' L^\beta\abs{\Tau^{i_*}_t}^{1 - (1 - \kappa)\beta}  + H \sqrt{ \abs{\Tau^{i_*}_t } \log(2/\delta) } + \RR_{i_*}^{\Pi_{i_*}} \sqrt{ \abs{\Tau^{i_*}_t}} \right)
}
for $C_\ZZ$ large enough. Therefore, it holds that $\sum_{t' \in \Tau^{i_*}_t} \widehat V_{j}^t - g_{t'} \leq \ZZ_{i_*}(\abs{\Tau^{i_*}_t}, \VV_j, \VV'_j)$.
\end{proof}
The main proposition states that a better instance-dependent rate is available under less restrictive assumptions on ``realizability" by utilizing the test based on the $V^*$ estimators.

\begin{proposition}
For a given $\MM$, let Assumption~\ref{asmpt::vstar-estimator} hold some for $\alpha, \beta$ and $i \geq i_*$ and let $\kappa \in (0, 1/2]$. Then, with probability at least  $1 - \delta'$, \ms{} in Algorithm~\ref{alg::ms} with the modified test (Equation~\ref{eqn:alt_test_V}) satisfies the regret bound
\eq{
& \widetilde O \left( HLT^{1 - \kappa} + \RR_{i_*}^{\Pi_{i_*}} \sqrt{LT} + \sum_{j < i_*} \Delta_{j, i_*}\max \left\{ \frac{L^{\frac{1}{1 - \kappa}} \VV_{i_*}^{\frac{1}{(1 - \kappa)\alpha}}}{\Delta_{j, i_*}^{\frac{1}{(1 - \kappa)\alpha}}},\ 
\frac{L^{\frac{1}{1 - \kappa}}{\VV'_{i_*}}^{\frac{1}{(1 - \kappa)\beta}}}{\Delta_{j, i_*}^{\frac{1}{(1 - \kappa)\beta}}}, \ 
\frac{(\RR_{i_*}^{\Pi_{i_*}} + H\log^{1/2}(LT/\delta'))^2}{\Delta_{j, i_*}^2} \right\}  \right)
}
\end{proposition}

\begin{proof}
As discussed previously, the sufficient events occur with probability at least $1 - \delta'$. Similar to Theorem~\ref{thm::id}, we now show that the gaps $\Delta_{j, i_*}$ can be bounded  by using the estimation error of $\hat V^{i_*}$ and the concentration bounds from $E'$. Let $t$ be such that $\hat\imath_t = j$ and the test succeeds. Then,
\eq{
\Delta_{j, i_*} & = V^* - V_j^* \\
            & \leq \widehat V_{i_*}^t + \frac{\VV_{i_*}}{\abs{\Tau^{i_*}_t}^\alpha}  + \frac{\VV_{i_*}'}{\abs{\Tau^{i_*}_t}^\beta}  - \frac{1}{\abs{\Tau^j_t}} \sum_{t'  \in \Tau^j_t} V^{\pi_{t'}} \\
             &\leq \widehat V_{i_*}^t + \frac{\VV_{i_*}}{\abs{\Tau^{i_*}_t}^\alpha}  + \frac{\VV_{i_*}'}{\abs{\Tau^{i_*}_t}^\beta} - \frac{1}{\abs{\Tau^j_t}} \sum_{t'  \in \Tau^j_t} g_{t'} + \frac{1}{\abs{\Tau^j_t}}\sum_{t'  \in \Tau^j_t} \epsilon_{t'} \\
             & \leq  C_\ZZ \left( \VV_{i_*} L^\alpha\abs{\Tau^j_t}^{- (1 - \kappa)\alpha} + \VV_{i_*}' L^\beta\abs{\Tau^j_t}^{- (1 - \kappa)\beta}  + H \sqrt{\frac{  \log(1/\delta) } {\abs{\Tau^j_t}} } + \frac{\RR_{i_*}^{\Pi_{i_*}}}{\sqrt{ \abs{\Tau^j_t}} }\right)  + \frac{\VV_{i_*}}{\abs{\Tau^{i_*}_t}^\alpha} + \frac{\VV_{i_*}'}{\abs{\Tau^{i_*}_t}^\beta}  +  H \sqrt{\frac{  \log(1/\delta) } {\abs{\Tau^j_t}} }
}  
Again noting that $\abs{\Tau^{i_*}_t} \geq \frac{t^{1 - \kappa}}{8 L } \geq \frac{\abs{\Tau^{j}_t}^{1 - \kappa}}{8 L }$, the above can be simplified to 
\eq{
\Delta_{j, i_*} & \leq C_\ZZ' \cdot \left( 2\VV_{i_*} L^\alpha\abs{\Tau^j_t}^{- (1 - \kappa)\alpha}  + 
2\VV_{i_*}' L^\beta\abs{\Tau^j_t}^{- (1 - \kappa)\beta}  +
\frac{ 2H \log^{1/2}(1/\delta) + \RR_{i_*}^{\Pi_{i_*}} }{ \abs{\Tau^j_t}^{1/2} } \right) \\
& \leq 6 C_\ZZ' \cdot \max \left\{ \frac{\VV_{i_*} L^\alpha}{ \abs{\Tau^j_t}^{(1 - \kappa)\alpha}}, \  \frac{\VV_{i_*}' L^\beta}{ \abs{\Tau^j_t}^{(1 - \kappa)\beta}}, \ 
\frac{ H \log^{1/2}(1/\delta) + \RR_{i_*}^{\Pi_{i_*}} }{ \abs{\Tau^j_t}^{1/2} } \right\}
}
where $C'_{\ZZ} = \max\{1, C_\ZZ\}$. Then, we can consider the three potential cases to upper bound $\abs{\Tau_t^j}$. Depending on the maximal term, one of the three possible cases occurs:
\eq{
\abs{\Tau_t^j} \leq \left( \frac{6C'_\ZZ \VV_{i_*} L^\alpha}{\Delta_{j, i_*}} \right)^{\frac{1}{(1 - \kappa) \alpha}}, & & 
\abs{\Tau_t^j} \leq \left( \frac{6C'_\ZZ \VV_{i_*}' L^\beta}{\Delta_{j, i_*}} \right)^{\frac{1}{(1 - \kappa) \beta}}, & & 
\abs{\Tau^j_t} \leq  \left( \frac{6C'_\ZZ (H \log^{1/2}(1/\delta) + \RR_{i_*}^{\Pi_{i_*}})}{\Delta_{j, i_*}} \right)^{2}
}

The regret during the misspecified phase becomes
\eq{
& \regret_{\tau_{\min}(\delta):\tau_*} \\
&= O \left(  HL T^{1 - \kappa} + Hi_* 
+ 
\RR_{i_*}^{\Pi_{i_*}}\sqrt{LT}
+
\sum_{j < i_*} \Delta_{j, i_*}\max \left\{ \frac{L^{\frac{1}{1 - \kappa}} \VV_{i_*}^{\frac{1}{(1 - \kappa)\alpha}}}{\Delta_{j, i_*}^{\frac{1}{(1 - \kappa)\alpha}}},\ 
\frac{L^{\frac{1}{1 - \kappa}}{\VV'_{i_*}}^{\frac{1}{(1 - \kappa)\beta}}}{\Delta_{j, i_*}^{\frac{1}{(1 - \kappa)\beta}}}, \ 
\frac{(\RR_{i_*}^{\Pi_{i_*}} + H\log^{1/2}(LT/\delta'))^2}{\Delta_{j, i_*}^2} \right\}
\right)
}
The total regret is
\eq{
& O \left(    
H L^{\frac{2}{1 - \kappa}} \log^{\frac{1}{1 - \kappa}}(1/\delta)
+ HL T^{1 - \kappa} + H i_*  \right) \\
& \quad +  O \left(\RR_{i_*}^{\Pi_{i_*}} \sqrt{LT} + 
\sum_{j < i_*} \Delta_{j, i_*}\max \left\{ \frac{L^{\frac{1}{1 - \kappa}} \VV_{i_*}^{\frac{1}{(1 - \kappa)\alpha}}}{\Delta_{j, i_*}^{\frac{1}{(1 - \kappa)\alpha}}},\ 
\frac{L^{\frac{1}{1 - \kappa}}{\VV'_{i_*}}^{\frac{1}{(1 - \kappa)\beta}}}{\Delta_{j, i_*}^{\frac{1}{(1 - \kappa)\beta}}}, \ 
\frac{(\RR_{i_*}^{\Pi_{i_*}} + H\log^{1/2}(LT/\delta'))^2}{\Delta_{j, i_*}^2} \right\}
\right) 
}
\end{proof}

Consider again the implications of this bound in the contextual bandit setting. It is possible that to estimate an upper bound of $V^*$ with rate $\widetilde O \left( \frac{ d^{1/4}_{j}}{n^{1/2}} + \frac{1}{n^{1/4}} \right)$,  where $n$ is the number of samples and $j \geq i_*$ \citep{foster2019model,kong2018estimating}. However, this would only give a one-sided estimation error bound. If a two-sided guarantee of the same form were possible, we would have $\alpha = 1/2$, $\beta = 1/4$, and $\VV_{i_*} = \widetilde O \left( d^{1/4} \right), \VV'_{i_*} = \widetilde O \left( 1 \right)$. We now state the following immediate corollary in this setting with constant gaps under the hypothesis that such an estimator for this problem exists and is given.

\begin{corollary}
For the linear contextual bandit problem under Assumption~\ref{asmpt::vstar-estimator} with constant gaps $\{\Delta_{j, i_*} \}_{j < i_*}$, let $\alpha = 1/2$, $\beta = 1/4$,  $\VV_{i_*} = \widetilde O(d^{1/4}_{i_*})$ and $\VV_{i_*}' = \widetilde O(1)$. Let the exploration parameter $\kappa = 1/2$. Then with probability at least $1 - \delta'$, \ms{} in Algorithm~\ref{alg::ms} with the modified test (Equation~\ref{eqn:alt_test_V}) satisfies the regret bound
\eq{
\widetilde O \left( \sqrt{T} + \sqrt{d_{i_*} T} + \sum_{j < i_*} \max \left\{ d_{i_*} \Delta^{-3}_{j, i_*}, \  \Delta^{-7}_{j, i_*}, \ d_{i_*} \Delta^{-1}_{j, i_*} \right\} \right) = \widetilde O \left( \sqrt{T} +  \sqrt{d_{i_*} T} + d_{i_*} \right)
}
where $\widetilde O$ hides dependence on the number of models $L$, the number of actions $K = \abs{\UU}$, and log factors.
\end{corollary}

For constant gaps, the scalings in $d$ and $T$ are nearly same for this estimator and the gap estimator of the previous section. The main difference arises in the dependence on the gap, $O(\Delta_{\min}^{-5})$ in this case compared to $O \left(\Delta_{\min}^{-2}  \right)$ in the previous case. In this case, it is clearly suboptimal.

\end{document}